\documentclass[twoside,11pt]{article}

\usepackage{amsthm}
\usepackage{jmlr2e}
\usepackage{mathextra}
\usepackage{pgfplots}
\usepackage{pgfplotstable}

\usepackage[capitalize]{cleveref}

\makeatletter
\def\cleartheorem#1{%
	\expandafter\let\csname#1\endcsname\relax
	\expandafter\let\csname c@#1\endcsname\relax
}
\makeatother

\cleartheorem{example}
\cleartheorem{theorem}
\cleartheorem{lemma}
\cleartheorem{proposition}
\cleartheorem{remark}
\cleartheorem{corollary}
\cleartheorem{definition}
\cleartheorem{conjecture}
\cleartheorem{axiom}
\newtheorem{example}{Example} 
\newtheorem{theorem}{Theorem}
\newtheorem{lemma}[theorem]{Lemma}

\newtheorem{corollary}[theorem]{Corollary}

\newtheorem{assumption}{Assumption}

\usepackage{booktabs}
\usepackage{tabularx}
\usepackage{enumerate}

\usepackage[algo2e,noline,noend,ruled]{algorithm2e}
\DontPrintSemicolon
\SetAlgoNoLine
\LinesNumbered
\let\oldnl\nl
\newcommand{\nonl}{\renewcommand{\nl}{\let\nl\oldnl}}
\SetKwInput{KwAssume}{Assume}

\usepackage[textsize=tiny,disable]{todonotes}
\newcommand{\todot}[2][]{\todo[size=\scriptsize,color=red!20!white,#1]{Tor: #2}}
\newcommand{\todoj}[2][]{\todo[size=\scriptsize,color=green!20!white,#1]{Johannes: #2}}

\newcommand{\R}{\mathfrak{R}}

\newcommand{\thetaopt}{\theta^*}

\newcommand{\hDelta}{\hat \Delta}
\newcommand{\htheta}{\hat \theta}

\newcommand{\IDS}{\text{\normalfont IDS}}

\newcommand{\dirac}[1]{e_{#1}}
\newcommand{\RRp}{\RR_{\geq 0}}
\newcommand{\opspan}[2]{\laspan(\im(#1) : #2)}
\newcommand{\PM}{\text{Proj}_{\vV}}

\newcommand{\Done}{\Delta_1}
\newcommand{\Dtwo}{\Delta_2}

\newcommand{\Ione}{I_1}
\newcommand{\Itwo}{I_2}

\newcommand{\pzero}{p_0}

\Crefname{algocf}{Algorithm}{Algorithms}
\crefname{algocf}{Algorithm}{Algorithms}
\crefname{assumption}{Assumption}{Assumptions}
\Crefname{assumption}{Assumption}{Assumptions}

\usepackage{lastpage}
\jmlrheading{24}{2023}{1-\pageref{LastPage}}{10/22}{8/23}{22-1248}{Johannes Kirschner, Tor Lattimore, Andreas Krause}
\ShortHeadings{title}{J. Kirschner, T. Lattimore, A. Krause}

\ShortHeadings{Linear Partial Monitoring}{Kirschner, Lattimore and Krause}
\firstpageno{1}

\begin{document}

\title{Linear Partial Monitoring for Sequential Decision Making\\Algorithms, Regret Bounds and Applications}

\author{\name Johannes Kirschner \email jkirschn@ualberta.ca \\
       \addr Department of Computer Science\\
       University of Alberta
       \AND
       \name Tor Lattimore \email lattimore@google.com \\
       \addr DeepMind\\
       London\\
       \name Andreas Krause \email krausea@ethz.ch \\
       \addr Department of Computer Science\\
       ETH Zurich
       \AND}

\editor{Aurelien Garivier}

\maketitle

\begin{abstract}
Partial monitoring is an expressive framework for sequential decision-making with an abundance of applications, including graph-structured and dueling bandits, dynamic pricing and transductive feedback models. We survey and extend recent results on the linear formulation of partial monitoring that naturally generalizes the standard linear bandit setting. The main result is that a single algorithm, information-directed sampling (IDS), is (nearly) worst-case rate optimal in all finite-action games. We present a simple and unified analysis of stochastic partial monitoring, and further extend the model to the contextual and kernelized setting.
\end{abstract}


\section{Introduction}

Partial monitoring \citep{rustichini1999games} is a flexible framework for stateless sequential decision making. The partial monitoring model captures the standard multi-armed and linear bandit setting, semi- and full information feedback models, dynamic pricing and variants of dueling bandits, to name just a few. Partial monitoring is formalized as a round-based game between a learner and an environment. In each round, the learner chooses an action and the environment provides a feedback. Further, there is an (unknown) reward associated to each action but, unlike in the bandit model, the reward is not necessarily directly observed. Instead, the learner is given a description of how the reward and feedback are related. In order to maximize the cumulative reward, the learner needs to find a careful balance between actions that lead to informative feedback and actions with high reward, which is the essence of the \emph{exploration-exploitation dilemma}. More specifically, the learner needs to collect data that allows it to identify an optimal action while minimizing the cost of playing sub-optimal actions relative to the optimal action (known as regret).

Whether the learner can succeed at this task depends on the structure of feedback and reward.  In this work, we focus on linearly correlated feedback models where, akin to the linear bandit model, the reward and observations are both described by linear features. Our formulation of the linear model includes the classical finite stochastic partial monitoring formulation as a special case, where reward and feedback are determined by finite matrices. 

Classifying games according to the difficulty of the exploration-exploitation trade-off has been an essential part of previous work on partial monitoring \citep{bartok2014partial,lattimore2019cleaning,kirschner2020pm}. By now it is understood that all finite-action linear partial monitoring games belong to one of four categories: \emph{trivial}, \emph{easy}, \emph{hard} and \emph{hopeless}, corresponding to $0$, $\tilde \oO(\sqrt{n})$, $\tilde \oO(n^{2/3})$ and $\Omega(1)$ regret in the worst-case over $n$ steps. The categories are defined by precise geometric conditions, which we will introduce in detail later.

Partial monitoring has a moderately well-deserved reputation for being a complicated framework. 
The complexity is largely a product of the generality and also the discrete nature of the standard setup. The latter leads to some
unfortunate practical problems. For example, the finite partial monitoring game that models a finite-armed Bernoulli bandit is exponentially
large in the number of actions. Although there now exist unified algorithms for finite partial monitoring, they cannot be practically
implemented on large games like this, enormously reducing  their applicability.


On the other hand, the \emph{linear} partial monitoring model that we present here is a \emph{natural extension} of the linear bandit setting, that only relaxes the requirement that the observation and reward features for any action are identical. By further introducing (convex) parameter constraints and using a suitable linear embedding, we completely recover the standard \emph{finite} partial monitoring setting.
At the same time, many problems are more naturally modelled in the linear framework. The classical $k$-armed Bernoulli bandit, 
for example, is modelled by a $k$-dimensional linear partial monitoring game.

The feedback/reward structure of partial monitoring requires exploration techniques that go beyond what is commonly used in bandit feedback models. In particular, it is known that optimism and Thompson sampling are insufficient to achieve sublinear regret in all types of partial monitoring games. Our presentation focuses on a different design principle, known as {\em information-directed sampling} \citep[IDS;][]{russo2014ids,kirschner2020pm}. This leads to a universal and practical algorithm for regret minimization in partial monitoring that provably works well in \emph{all} games. 

Compared to the previous work, we introduce several innovations, including constrained parameter sets and computationally efficient approximations. Distinct from all previous work, our formulation \emph{unifies} the finite and linear stochastic partial monitoring settings, thereby providing a unified approach that is \textit{(a)} nearly worst-case optimal in all possible games and \textit{(b)}  computationally efficient and practical. We present a concise and self-contained analysis, where nearly all geometric constructions only appear in the proof and not the algorithm. 

Last but not least, we present two extensions of linear partial monitoring that are of practical importance. The first is a contextual formulation that directly generalizes the standard contextual (stochastic) linear bandit setting. The second is a kernelized formulation, where the reward function is in a reproducing kernel Hilbert space. We provide direct extensions of IDS to both settings, demonstrating the versatility of the framework. A more detailed historical and contemporary account of related work is deferred to \cref{sec:related}.

\paragraph{Notation} 
For an integer $i \geq 1$, $[i] \eqdef \{1, \dots, i \}$. The $d$-dimensional identity matrix is $\eye_d$. 
For a compact set $\xX$, $\sP(\xX)$ is the set of distributions on $\xX$ with the Borel $\sigma$-algebra. For a measurable function $F:\xX \rightarrow \RR$ and a distribution $\mu \in \sP(\xX)$, we use $F(\mu) = \EE_{x \sim \mu}[F(x)]$ to denote the expectation. For a vector $v \in \RR^d$, the Euclidean norm is $\| v \|_2$. For a positive semi-definite matrix $A \in \RR^{d \times d}$, $\lambda_{\max}(A)$ is the largest eigenvalue, and $\|A\| = \lambda_{\max}(A)$ is the spectral norm. The associated norm to $A$ is $\|b\|_A = \sqrt{b^\top A b}$, where $b \in \RR^d$. This notion is generalized for rectangular matrices $B \in \RR^{d \times s}$, $\|B\|_A \eqdef \lambda_{\max}(B^\T A B)^{1/2}$. For two positive semi-definite matrices $A_1, A_2$, $A_1 \preccurlyeq A_2$ means that $A_2 - A_1$ is positive semi-definite. 
Given a set $X \subset \RR^d$, we write $\dim(X)$ for the dimension of the affine hull of $X$. We use the convention that $1/0 = \infty$, $0/0 = 0$ and $0 \cdot \infty = \infty$.

\section{Linear Partial Monitoring}
A \emph{linear partial monitoring game} is defined by a quadruple $(\aA, (\phi_a)_{a \in \aA}, (M_a)_{a \in \aA}, \Theta)$. Here, $\aA$ is a set of \emph{actions} that is used to index \emph{reward features} $\phi_a \in \RR^d$ and \emph{linear observation maps} $M_a \in \RR^{m \times d}$ for $a \in \aA$. $\Theta \subset \RR^d$ is a (convex) parameter set. The actions $\aA$, features $\phi_a$, observation maps $M_a$ and parameter set $\Theta$ are known to the learner.

\looseness -1 An \emph{instance} of a partial monitoring game is defined by a fixed parameter $\theta^* \in \Theta$ that is \emph{unknown} to the learner. At time $t \in [n]$, the learner chooses an action $a_t \in \aA$ according to a policy $\pi$ and observes an m-dimensional vector $y_t = M_{a_t} \theta^* + \epsilon_t \in \RR^m$, where $\epsilon_t$ is a zero-mean noise vector. As usual performance is measured by the \emph{expected regret} at time $n$,
\begin{align*}
	\R_n(\pi, \theta^*) = \max_{a \in \aA} \EE[\sum_{t=1}^n \ip{\phi_{a} - \phi_{a_t}, \theta^*}]
\end{align*}
where the expectation is over the randomness in the feedback and any randomization used by the policy. For simplicity we write $\R_n = \R_n(\pi, \theta^*)$ when the policy and the parameter are clear from the context. We emphasize that the \emph{only} difference to the linear bandit model is that \emph{reward and observation features are decoupled}, whereas in the bandit model $\phi_a = M_a^\T$ for all $a \in \aA$. The usefulness of decoupling reward and observation will become apparent in the examples presented in \cref{ss:examples}.

The $\sigma$-algebra generated by the history up to round $t$ is $\fF_t = \sigma(a_1,y_1,\ldots,a_t,y_t)$ and
the corresponding filtration over $n$ rounds is $(\fF_t)_{t=1}^n$.
We define the shorthand $\PP_t[\cdot] = \PP[\cdot | \fF_{t-1}]$ and $\EE_t[\cdot] = \EE[\cdot | \fF_{t-1}]$. The best action for a parameter $\theta \in \Theta$ is  $a^*(\theta) = \argmax_{a \in \aA} \ip{\phi_{a},\theta}$, chosen arbitrarily if it is not unique. The maximum gap under the true parameter is denoted by $\Delta_{\max} = \max_{a,b \in \aA} \ip{\phi_a - \phi_b, \theta^*}$.

We further require the following standard boundedness and noise assumptions.

\begin{assumption}[Convexity and Boundedness]\label{ass:bounded}
	The action set $\aA$ is compact and the parameter set $\Theta$ is convex. Further, we assume that the learner has access to bounds $\|M_a\|_2 \leq L$, $\|\phi\|_2 \leq 1$ and $\|\theta - \theta_0\|_2 \leq B$ for all $\theta \in \Theta$ and a given prior estimate $\theta_0 \in \Theta$.
\end{assumption}
Note that assuming $\|\phi\|_2 \leq 1$ is without loss of generality, as we can jointly re-scale $\phi_a$, $M_a$ and $\Theta$ to satisfy the condition while keeping the rewards and observations the same.

\begin{assumption}[Sub-Gaussian Noise]
	The noise vector $\epsilon_t \in \RR^m$ is $\rho$-sub-Gaussian, i.e.,
	\begin{align*}
		\forall u \in \RR^m\,,\quad \EE_t[\exp(u^\T \epsilon_t)] \leq \exp\left(\frac{\|u\|^2 \rho^2}{2}\right)\,.
	\end{align*}
\end{assumption}
Note that we allow the noise distribution to be a function of the action and the parameter, whereas the sub-Gaussian variance $\rho^2$ is a uniform constant known to the learner.

\subsection{Examples} \label{ss:examples}

The framework of linear partial monitoring captures a large variety of models for online decision-making, many of which have been studied independently in the literature. We provide some examples below.

\begin{example}[Linear Bandits]\label{ex:lb}
	As was already mentioned, the linear bandit model is recovered with $a = \phi_a = M_a^\T$ for an action set $\aA \subset \RR^d$. A direct extension are heteroscedastic linear bandits, where the noise distribution depends on the chosen action. In the setting studied by \citet{kirschner2018heteroscedastic}, the learner observes $y_t = a_t^\T \theta^* + \epsilon_t$, where $\epsilon_t$ is $\rho(a_t)$-sub-Gaussian and the \emph{noise function} $\rho :\aA\rightarrow \RRp$ governs the signal-to-noise ratio of the associated observation.  To formulate the heteroscedastic setting in linear partial monitoring we can define $\phi_a = a$, $M_a = a/\rho(a)$ and $\rho=1$. 
\end{example}


\begin{example}[Graph-Structured Feedback]\label{ex:pm-graph}
	Semi-bandit feedback or \emph{side-observations} refer to models between full information and bandit feedback. Semi-bandit feedback can be specified with a \emph{feedback graph} defined over actions \citep{mannor2011bandits,caron2012leveraging}. When choosing an action, the learner observes the reward of all adjacent actions. Formally, assume that $\wW \subset \aA \times \aA$ is a set of (directed) edges. For each $a \in \aA$, the feedback map is defined to reveal the reward of all adjacent actions,
	\begin{align*}
		M_a = [\phi_c  : c\in \aA \text{ s.t. } (a, c) \in \wW]^\T\,,
	\end{align*}
	with the minor technicality that the observation dimension now in general depends on the action. This can be resolved by padding the observation matrices with zero vectors, or modifying the model to allow action-dependent observation dimension. If we explicitly require that the edge set contains all self-loops $(a,a) \in \wW$ for all $a \in \aA$, then the graph feedback structure naturally interpolates between bandit feedback (empty graph, with self-loops) and full information (fully connected graph). We point out that \citet{liu2018information} studied a version of IDS in the corresponding Bayesian regret setting. 
\end{example}

\begin{example}[Dueling Bandits] \label{ex:dueling} 
	In dueling bandits, the learner chooses pairs of actions and receives noisy feedback about which of the two actions has higher reward \citep{yue2012k,sui2018advancements}. The dueling bandit model has many intricate variants and a vast literature on its own \citep{bengs2021preference}. Not every dueling bandit is easily modeled as a partial monitoring game. 
	With additional assumptions, \citet{gajane2015utility} demonstrated that the \emph{utility-based} dueling bandit problem can be formulated as a partial monitoring game. Here we focus on a similar setup with quantitative feedback on the reward difference of the chosen actions, opposed to the more common binary signal. The quantitative variant has received relatively little attention in the literature and has some interesting applications in robust regret minimization \citep{kirschner2020dueling}.
	
	Formally, let $\iI$ be a ground set of actions with an associated feature mapping $\phi : \iI \rightarrow \RR^d, a \mapsto \phi_a$. The action set is $\aA = \iI \times \iI$. For any ${(a,b) \in \aA}$, we define \emph{utility-based dueling feedback} by the feedback map $M_{a,b} = \phi_a - \phi_b$. Hence, when choosing the action pair $(a_t,b_t)$, the learner observes the noisy reward difference $y_t = \ip{\phi_{a_t} - \phi_{b_t}, \theta} +\epsilon_t$, while collecting (unobserved) reward for both actions with reward features $\phi_{a,b} = \phi_a + \phi_b$. We remark that the sub-Gaussian likelihood combined with appropriate boundedness of the reward includes the standard binary feedback model with $y_t \in \{-1,1\}$ as a special case.
\end{example}

\begin{example}[Graph-Structured Dueling Bandits]\label{ex:graph-dueling}
	We propose a novel variant of the dueling bandit that extends \cref{ex:pm-graph} with a feedback graph structure. Concretely, the learner has access to dueling reward features $\phi_{a,b} = \phi_{a} + \phi_{b}$, and dueling feedback $M_{a,b} = \phi_{a} - \phi_{b}$, but only for pairs $(a,b) \in \aA \subset \iI \times \iI$. We can think of $\aA$ as a set of edges on a graph over the base indices $\iI$. In other words, the learner can only observe a dueling action $(a,b)$ if there is a directed edge from $a$ to $b$. To learn the reward difference between actions that are not neighbors, the learner has to combine dueling observations from a path that connects the two actions. For general action features, the learner can only hope to learn all reward differences if the graph is connected.
\end{example}

\begin{example}[Combinatorial Partial Monitoring]
	In the \emph{combinatorial} bandit problem the action set can be exponentially large. Learning algorithms for this scenario are designed to only use the solver for the offline problem, i.e., assuming access to an oracle that solves $\argmax_{a \in \aA} \ip{\phi_a, \theta}$ for any $\theta \in \mM$. While the IDS algorithm we introduce in Section~\ref{sec:ids} is \emph{not} oracle efficient, the theory still applies. The combinatorial setting is the motivation for the linear partial monitoring setting in the work by \citet{lin2014combinatorial} and \citet{chaudhuri2016phased}. Both previously proposed methods are oracle efficient, however suffer sub-optimal regret in locally observable games.
	
	The combinatorial version of the multi-armed bandit setting \citep{cesa2012combinatorial} makes more specific assumptions on the feedback structure. Let $\iI$ be an index set with associated features $\phi_a$ for $a \in \aA$. The action set $\aA \subset 2^\iI$ consists of subsets of $\iI$.  The reward for choosing an action $a \in \aA$ is the sum of rewards $f_\theta(a) = \sum_{i \in a} \ip{\phi_i, \theta}$. Equivalently, the features for action $a$ are $\sum_{i \in a} \phi_i$. 
	Two variants for the feedback maps are commonly considered: i) bandit feedback, that is $M_a = \phi_a$, and ii) semi-bandit feedback, $\mM_a = [\phi_i : i \in a]^\T$. 
	An important special case is the batch setting where the learner chooses $m$ actions at once, i.e., $\aA = \{ a \subset \iI : |a| = m\}$. Note that the example exhibits an exponential blow-up of the action set in the batch size $m$. 
\end{example}

\begin{example}[Transductive Bandits]\label{ex:pm-transductive}
	In the transductive bandit setting, the learner obtains informative feedback only on a set of actions that is dedicated for exploration. At the same time, the objective is to achieve low regret on a different target set of actions, that when played, do not reveal information. The setting was  proposed by \citet{fiez2019transductive} in the context of best arm identification, and we refer the reader to this work for more examples. A toy example from the partial monitoring literature that fits into this category is that of  \emph{``apple tasting''} \citep{cesa2006regret}. In each round, the learner is presented an apple, and decides whether to taste it. Tasting determines if the apple is rotten or not. Apples that have been tasted cannot be sold anymore and incur a fixed cost. Not tasting the apple comes with the risk of selling a rotten apple, which also incurs a cost but is not observed. We remark that this setup is closely related to \emph{label efficient prediction} \cite{cesa2005minimizing}. \looseness=-1 
\end{example}

\begin{example}[Finite Partial Monitoring]\label{ex:finite}
\newcommand{\xdistr}{\vartheta}

Most previous work has focused on the \emph{finite partial monitoring} setting that we briefly introduce now. In addition to a finite set of \emph{actions} $\aA = [k]$, a finite partial monitoring game consists of a set of \emph{signals} $\Sigma = [m]$ used for the feedback and a finite set of \emph{outcomes} $\xX= [d]$ that determines reward and feedback for each action. 
Reward\footnote{The finite setting is often formulated with losses instead of rewards. For consistency with our presentation, we use the equivalent setup with rewards. The loss formulation is easily recovered by flipping the sign of the feature vectors.} and feedback are defined by,
\begin{enumerate}[i)]
	\item a \emph{reward function} $R: \aA \times \xX \rightarrow [0,1]$; and
	\item a \emph{signal function} $\Phi : \aA \times \xX \rightarrow \Sigma$. 
\end{enumerate}
The learner has access to both $R$ and $\Phi$. In each round $t=1,\dots,n$ of the game, the learner chooses an action $a_t \in \aA$. In the stochastic version of the problem, the outcome $x_t$ is sampled from an unknown and fixed distribution $\xdistr \in \sP(\xX)$. The learner observes a signal $\sigma_t = \Phi(a_t, x_t) \in \Sigma$ and obtains reward $R(a_t, x_t)$. Neither the reward nor the outcome is revealed to the learner.  \looseness=-1

As in previous work \cite[e.g.,][]{bartok2014partial}, we use vector notation to describe the finite setting in the linear framework. Let $e_a \in \RR^k$, $e_x \in \RR^d$ and $e_{\sigma} \in \RR^m$ be the basis vectors corresponding to action $a \in \aA$, outcome $x \in \xX$ and signal $\sigma \in \Sigma$. We use $R \in \RR^{k \times d}$ as a matrix and function interchangeably, such that $e_a^\T R e_x = R(a,x)$. Further, we introduce reward features $\phi_a = R^\T e_a \in \RR^d$, defined as the row of $R$ corresponding to action $a$. For each action $a \in \aA$, the observation matrix $S_a \in \{0,1\}^{m \times d}$ is such that $e_\sigma^\T S_a e_x = \chf{\Phi(a,x) = \sigma}$. We use the symbol $S_a$ instead of $M_a$ to emphasize the particular structure of the feedback map.  The distribution $\xdistr \in \Theta = \sP(\xX)$ is identified with a vector in the $(d-1)$-dimensional probability simplex. In particular, $S_a \xdistr$ is the distribution over 
the observed signals for action $a \in \aA$. If the learner chooses action $a_t \in \aA$  in round $t$, and the outcome is $x_t \in \xX$, then the corresponding observation vector is $y_t = e_{\sigma_t} = S_{a_t} e_{x_t} \in \RR^m$. 

Let $\xi_t = S_{a_t}(e_{x_t} - \xdistr)$ and note that $\EE_t[\xi_t] = 0$. The observation can be written as $y_t = S_{a_t} \xdistr + \xi_t$. One directly verifies that for any unit vector  $u \in \RR^s$, $|u^\T \xi_t| \leq \|u\|_\infty \|\xi_t\|_1 \leq \|S_a e_{x_t} - S_a\xdistr\|_1 \leq 2$. Hence $\xi_t$ is a 4-sub-Gaussian random vector in $\RR^m$.
\end{example}

We note that the difficulty of the games depends on the parameter set  $\Theta$. In particular, adding constraints to $\Theta$ can make the game much easier. The next example shows that a good algorithm for finite partial monitoring has to reason about the parameter set $\Theta$.

\begin{example}[Finite vs Linear Partial Monitoring]\label{ex:linear-counter}
	Consider the finite game defined by reward and signal matrices
	\begin{align*}
		R = 
		\begin{pmatrix}
			1 & 1 \\
			0 & 0  \\
		\end{pmatrix}\,,
		\qquad
		\Sigma
		= 
		\begin{pmatrix}
			0 & 0 \\
			0 & 0 
		\end{pmatrix}\,.
	\end{align*}
	The signal matrix uses only one symbol, therefore the learner cannot distinguish the outcomes. However, when $\Theta$ is the probability simplex, the rewards are such that the first action is always optimal, so in finite partial monitoring a good algorithm has zero regret. On the other hand, when $\Theta = \{\theta \in \RR^d :\|\theta\|_2 \leq 1 \}$ the learner has to consider the case when $\theta = (-\sqrt{2}, -\sqrt{2})$, and the second action is optimal. Consequently, any learner suffers linear regret on at least one of the two cases. 
\end{example}

We conclude the section with a two more examples of in the language of finite partial monitoring.

\begin{example}[Multi-Armed Bandits]
	In games with bandit information, the learner observes the reward of each action $a \in \aA$ by playing it. Since we allow only finitely many signals, the reward of each arm is also one of finitely many values. For Bernoulli bandits with $k$ arms specifically, $\aA = [k]$, $\Sigma = \{0,1\}$ and $\xX = \{0,1\}^k$. The reward and feedback functions are
	\begin{align*}
		R(a,x) = \Phi(a, x) = x_a\,.
	\end{align*}
	A consequence of the finite partial monitoring setup is that the parameter dimension $d = 2^k$ is exponentially large in the number of arms. On the other hand, we will see in section \cref{ss:ids-ball} that all relevant information is contained in a $2k$-dimensional subspace of $\RR^d$.
\end{example}

\begin{example}[Dynamic Pricing]
	\newcommand{\Yes}{\text{\textsc{y}\!\!}}
	\newcommand{\No}{\text{\textsc{n}\!\!}}
	One of the most notable applications of finite partial monitoring is dynamic pricing. This game is between a seller and a potential customer. The learner takes the role of the seller with the goal to optimally price a product. The action and outcome sets are a (discrete) set of prices corresponding to an offer and the price the customer is willing to pay, e.g.~$\aA = \xX = \{\$1,\, \$2,\, \$3\}$.  The feedback is whether the customer buys the product ($a \leq x$, $\Phi(a, x) = \Yes\,\,$), or not ($a > x$, $\Phi(a,x) = \No\,\,$). The reward consists of a fixed opportunity cost $c > 0$ and the difference between the offer and the price the customer would have payed, $R(a, x) = (a-x)\chf{a \leq x}-c\chf{a > x}$. With $c = 2$ and $\xX$, $\aA$ as above, the corresponding loss and signal matrices are:
	\begin{align*}
		R &= 
		\begin{pmatrix}
			0 & -1 & -2 \\
			-2 & 0 & -1 \\
			-2 & -2 & 0 \\
		\end{pmatrix} &
		\Phi &=
		\begin{pmatrix}
			\Yes & \Yes & \Yes\,\,\, \\
			\No & \Yes & \Yes\,\,\,\\
			\No & \No & \Yes \,\,\, \\
		\end{pmatrix}
	\end{align*}
\end{example}

%
%
%

%

\section{Information-Directed Sampling for Linear Partial Monitoring}\label{sec:ids}
\begin{algorithm2e}[t]
	\KwIn{Action set $\aA$, gap estimate $\hat \Delta_t : \aA \rightarrow \RRp$, information gain $I_t : \aA \rightarrow \RRp$}
	\caption{Information-Directed Sampling} \label{alg:ids}
	\For{$t=1,2,3, \dots, n$}{
		%
		$\displaystyle \mu_t \gets \argmin_{\mu \in \sP(\aA)} \left\{\Psi_t(\mu) = \frac{\hat \Delta_t(\mu)^2}{I_t(\mu)}\right\}$\tcp*{IDS distribution}
		Sample $a_t \sim \mu_t$, observe feedback $y_t$\;
	}
\end{algorithm2e}

Information-directed sampling (IDS) is a \emph{design principle} that requires user choices and leads to different algorithms in different settings. 
As presented in \cref{alg:ids}, IDS is abstractly defined for a sequence of \emph{gap estimates} $\hat \Delta_t : \aA \rightarrow \RRp$ and an \emph{information gain} functions $I_t : \aA\rightarrow \RRp$. We present concrete choices for both quantities shortly. 
One should think of $\hat \Delta_t(a)$ as an estimate of the instantaneous regret of playing action $a$ and $I_t(a)$ as some measure of the information gained when
playing action $a$. 
Naturally, the gap estimates and information gain are computed using observations from previous rounds and are therefore predictable with respect to the filtration $(\fF_t)_{t=1}^n$.  We also assume that $I_t$ is not zero for at least one action.

Recall that for a distribution $\mu \in \sP(\aA)$, we denote $\Delta_t(\mu) = \EE_{a \sim \mu}[\Delta_t(a)]$ and $I_t(\mu) = \EE_{a \sim \mu}[I_t(a)]$. The IDS distribution $\mu_t$ is defined as the minimizer of the ratio between squared expected regret and expected information gain:
\begin{align}
	\mu_t = \argmin_{\mu \in \sP(\aA)} \left \{\Psi_t(\mu) \eqdef \frac{\hat \Delta_t(\mu)^2}{I_t(\mu)} \right\}\,. \label{eq:ids-def}
\end{align}
The objective $\Psi_t(\mu)$ is called the \emph{information ratio} of the sampling distribution $\mu \in \sP(\aA)$.
The minimizer always exists for compact $\aA$ (\cref{lem:ratio-existence}). IDS is defined as the policy that samples $a_t \sim \mu_t$ in round $t$. 
Intuitively, to achieve a small information ratio, the learner has to sample actions from a distribution with small expected (estimated) regret or large information gain. This intuition will appear formally in the proofs in \cref{sec:proofs}.

The information ratio  $\Psi_t$ satisfies  several important properties, which allow us to solve the optimization problem \eqref{eq:ids-def} efficiently \citep{russo2014ids}. First, the function $\mu \mapsto \Psi_t(\mu)$ is convex for any choice of $\hat \Delta_t : \aA \rightarrow \RRp$ and $I_t : \aA \rightarrow \RRp$ (\cref{lem:ratio-convexity}). Further, the minimizing distribution $\mu_t$ can always be chosen with a support of at most two actions (\cref{lem:ratio-support}). Using this property, and provided that $\hat \Delta_t(a)$ and $I_t(a)$ have been computed for all $a \in \aA$, we can find the exact IDS distribution by enumerating all pairs of actions and solving the trade-off for each pair in closed-form (\cref{lem:ratio-closed-form}). 

\paragraph{Approximate IDS} It is also possible to obtain a $\frac{4}{3}$-approximation of \cref{eq:ids-def} in $\oO(|\aA|)$ time with oracle access to $\hat \Delta_t(a)$ and $I_t(a)$. To do so, we first find an action that minimizes the gap estimates, $\hat a_t = \argmin_{a \in \aA} \hat \Delta_t(a)$. We then optimize the trade-off between $\hat a_t$ and some other action $b \neq \hat a_t$, i.e.
\begin{align}
	p_t, b_t = \argmin_{p \in [0,1], b \in \aA} \Psi_t\big((1-p)\dirac{\hat a_t} + p \dirac{b}\big)\,.\label{eq:approximate-IDS}
\end{align}
As before, for fixed $\hat a_t,b \in \aA$, the optimization problem can be solved in closed form (\cref{lem:ratio-closed-form}), and it remains to iterate over the action set to find the best alternative action $b_t$. The distribution $\tilde \mu_t = (1-p_t)\dirac{\hat a_t} + p_t \dirac{b_t}$ satisfies $\Psi_t(\tilde \mu_t) \leq \frac{4}{3} \Psi_t(\mu_t)$ (\cref{lem:ratio-approximate}). 

\paragraph{General Regret Bounds} A few more things can be said without committing to specific choices of the gap estimate and information gain. The information ratio $\Psi_t$ appears as a central quantity in the regret analysis. To understand how, consider any adaptive policy $\pi_n = (\mu_t)_{t=1}^n$. We first bound the sum over the gap estimates:
\begin{align}
	\EE[\sum_{t=1}^n \hDelta_t(a_t)] &= \EE[\sum_{t=1}^n \hDelta_t(\mu_t)] = \EE[\sum_{t=1}^n \sqrt{\Psi_t(\mu_t) I_t(\mu_t)}]\nonumber\\
	&\leq \sqrt{\EE[\sum_{t=1}^n \Psi_t(\mu_t)] \EE[\sum_{t=1}^n I_t(a_t)]} \,. \label{eq:ids-1}
\end{align}
The first equality uses the tower rule, 
$\EE\big[\hDelta_t(a_t)\big] = \EE\big[\EE_t{\big[}\hDelta_t(a_t){\big]}\big] = \EE{\big[}\hDelta_t(\mu_t){\big]}$. 
The second equality uses the definition of the information ratio, and \cref{eq:ids-1} follows from the Cauchy-Schwarz inequality and another application of the tower rule. Note that IDS is the policy that myopically minimizes the first sum in the upper bound. The second sum is the \emph{total information gain}, which we abbreviate with
\begin{align}
	\gamma_n \eqdef \sum_{t=1}^n I_t(a_t)\,.\label{eq:total-info-def}
\end{align}
For the regret $\R_n = \EE[\sum_{t=1}^n\Delta(a_t)]$, \cref{eq:ids-1,eq:total-info-def} imply
\begin{align}
		\R_n \leq \sqrt{ \EE[\sum_{t=1}^n \Psi_t(\mu_t)] \EE[\gamma_n]} + \EE[\sum_{t=1}^n \Delta(a_t) - \hDelta_t(a_t)] \,.\label{eq:ids-2}
\end{align}
In the frequentist IDS framework, the gap estimate $\hDelta_t(a)$ is chosen as a high-probability upper bound on the true gap $\Delta(a)$. This way, the estimation error (the second term in the last display) contributes only negligibly to the overall regret. The total information gain $\gamma_n$ can be interpreted as a surrogate of the sample complexity of identifying the best action. In the finite-dimensional linear setting $\gamma_n$ depends only logarithmically on the horizon. Lastly, if $\Psi_t(\mu_t) \leq \alpha$ almost surely for all $1 \leq t \leq n$, then, by construction, the IDS policy has regret at most
\begin{align*}
	\R_n \leq \sqrt{n \alpha \EE[\gamma_n]} + \sum_{t=1}^n \EE[\Delta(a_t) - \hDelta_t(a_t)]\,.
\end{align*}
The usefulness of the bound stems from the fact that we can analyze IDS by explicitly designing sampling distributions that achieve a small information ratio, without having to know the exact behavior of the IDS algorithm. Further note that any constant approximation of the IDS distribution directly translates to the regret bound, which allows us to use the approximation in \cref{eq:approximate-IDS}.

In the next theorem, we summarize the result in slightly generalized form. We make use of the \emph{generalized information ratio} introduced by \citet{lattimore2020mirror},
\begin{align}
	\Psi_{\kappa,t}(\mu) \eqdef \frac{\hat \Delta_t(\mu)^\kappa}{I_t(\mu)}\,.\label{eq:generalized-ratio}
\end{align}
Note that the previous definition is recovered with $\kappa=2$, i.e.~$\Psi_{2,t} = \Psi_t$. 
The next lemma states the regret bound \eqref{eq:ids-2} for the generalized information ratio. A similar bound on the Bayesian regret is given by \citet[Theorem 4]{lattimore2020mirror}. 
\begin{theorem}\label{thm:ids-regret-general}
	Assume that $\Psi_{\kappa,t}(\mu_t) \leq \alpha_t$ holds almost surely for an $\fF_t$-predictable sequence $(\alpha_t)_{t=1}^n$,
  and let $\bar \alpha_n = \frac{1}{n}\sum_{t=1}^n \alpha_t$. Then
	\begin{align*}
		\R_n \leq (\EE[\bar \alpha_n]\EE[\gamma_n])^{\frac{1}{\kappa}} n^{1 - \frac{1}{\kappa}} + \sum_{t=1}^n \EE[\Delta(a_t) - \hat \Delta_t(a_t)]\,.
	\end{align*}
\end{theorem}
\begin{proof}
	We bound the estimated regret similarly as before:
	\begin{align*}
		\EE[\sum_{t=1}^n \hat \Delta(a_t)] = \EE[\sum_{t=1}^n (\Psi_{\kappa,t}(\mu_t) I_t(\mu_t))^{1/\kappa}] &\stackrel{(i)}{\leq} \EE[\sum_{t=1}^n \Psi_{\kappa,t}(\mu_t)^{\frac{1}{\kappa -1}}]^{1 - \frac{1}{\kappa}} \EE[\sum_{t=1}^n I_t(a_t)]^{\frac{1}{\kappa}}\\
		&\stackrel{(ii)}{\leq} \EE[\sum_{t=1}^n \alpha_t^{\frac{1}{\kappa -1}}]^{1 - \frac{1}{\kappa}} \EE[\gamma_n]^{\frac{1}{\kappa}}\\
		&\stackrel{(iii)}{\leq} \EE[ \Big(\sum_{t=1}^n \alpha_t\Big)^{\frac{1}{\kappa-1}} n^{\frac{\kappa-2}{\kappa-1}}]^{1 - \frac{1}{\kappa}} \EE[\gamma_n]^{\frac{1}{\kappa}}\\
		&\stackrel{(iv)}{\leq} n^{\frac{\kappa-1}{\kappa}} \EE[\bar \alpha_n]^{\frac{1}{\kappa}}  \EE[\gamma_n]^{\frac{1}{\kappa}}\,.
	\end{align*}
	We used $(i)$ and $(iii)$: Hölder's inequality, $(ii)$: definitions of $\alpha_t$ and $\gamma_n$, and $(iv)$: Jensen's inequality and the definition of $\bar \alpha_n$. The claim follows by introducing the estimation error.
\end{proof}

We remark that the generalized information ratio $\Psi_{\kappa,t}$ appears naturally in the analyses of games where the minimax regret is of order $\oO(n^{\frac{\kappa-1}{\kappa}})$ (see \cref{sec:proofs}). In such cases, it is natural to use an algorithm that optimizes $\Psi_{\kappa,t}$ directly.

The following lemma shows the perhaps surprising result that the IDS distribution obtained as a minimizer of $\Psi_{2,t}(\mu)$ approximately minimizes $\Psi_{\kappa,t}(\mu)$ for any $\kappa \geq 2$. This justifies the use of the $\Psi_{2,t}$ information ratio even in cases when a $\sqrt{n}$ regret rate is not attainable. \looseness=-1
\begin{lemma}[{\citet[Lemma 21]{lattimore2020mirror}}]\label{lem:psi-kappa-bound}
	Let $\mu_t = \argmin_{\mu \in \sP(\aA)}\Psi_{2,t}(\mu)$ be the IDS distribution computed for $\Psi_{2,t}$. Then for all $\kappa \geq 2$,
	\begin{align*}
		\Psi_{\kappa,t}(\mu_t) \leq 2^{\kappa -2} \min_{\mu \in \sP(\aA)}\Psi_{\kappa,t}(\mu)\,.
	\end{align*}
\end{lemma}

%
%

\subsection{IDS for Linear Partial Monitoring}\label{ss:ids-ball}
We now propose natural choices for the gap estimate and the information gain functions for linear partial monitoring games $\gG = (\aA, (\phi_a)_{a \in \aA}, (M_a)_{a \in \aA}, \Theta)$. The definitions and analysis follow the ideas of \citet{kirschner2020pm}, with a few innovations to account for constrained parameter sets $\Theta$. We note that parameter constraints are required to recover optimal bounds in the finite partial monitoring setting. The complete approach is summarized in \cref{alg:ids-pm}. Various improvements and extensions are discussed later in \cref{sec:extensions}.
\begin{algorithm2e}[t]
	\KwIn{Action set $\aA$, parameter set $\Theta$, feature maps $\phi_a$, feedback maps $M_a$, basis $W$, regularizer $\lambda > 0$, prior estimate $\theta_0$, norm bound $B > 0$, noise variance $\rho^2$.}
	\caption{IDS for Linear Partial Monitoring} \label{alg:ids-pm}
	\For{$t=1,2,3, \dots, n$}{
		$\htheta_t \gets \argmin_{\theta \in \Theta} \sum_{s=1}^{t-1} \|M_{a_s}  \theta - y_s\|^2 + \lambda \|\theta - \theta_0\|^2$ \tcp*{solve least-squares}
		$V_t \gets \sum_{s=1}^{t-1} M_{a_s}M_{a_s}^\T + \lambda \eye_d\,,\,\, W_t \gets W^\T V_t W$\;
		$\beta_{t,\delta}^{1/2} \gets \rho\sqrt{\log \det(W_t)  - \log \det(\lambda \eye_d)+ 2 \log(1/\delta)} + \sqrt{\lambda}B$\;
		$\eE_t \gets \{\theta \in \Theta : \|\theta - \hat \theta_t\|_{V_t}^2 \leq \beta_{t,1/t^2}\}$ \tcp*{confidence set}
		$\hat \Delta_{t}(a) \gets \max_{\theta \in \eE_t} \max_{b \in \aA} \ip{\phi_b - \phi_a, \theta}$\tcp*{gap estimates}
		$I_t(a) \gets \frac{1}{2}\log\det \left(\eye_m + M_a V_t^{-1} M_a^\T\right)$ \tcp*{information gain}
		$\mu_t \gets \argmin_{\mu \in \sP(\aA)} \dfrac{\hat \Delta_t(\mu)^2}{I_t(\mu)}$ \tcp*{IDS distribution}
		$a_t \sim \mu_t$\;
		Choose $a_t$, observe $y_t = \ip{M_{a_t}, \theta} + \epsilon_t$\;
	}
\end{algorithm2e}

\paragraph{Parameter Estimate} We assume prior knowledge of some $\theta_0 \in \Theta$ such that $\|\theta_0 - \theta^*\|_2 \leq B$, where both $\theta_0$ and $B$ are known to the learner (\cref{ass:bounded}). As a main tool for estimating $\theta^*$, we rely on regularized linear least-squares on $\Theta$,
\begin{align}
	\hat \theta_t = \argmin_{\theta \in \Theta} \sum_{s=1}^{t-1} \|M_{a_s}  \theta - y_s\|^2 + \lambda \|\theta - \theta_0\|^2 \,, \label{eq:ls}
\end{align}
where $\lambda \in \RRp$ is a regularizer. For the case where $\Theta = \RR^d$, the usual closed-form is available, $\hat \theta_t  = V_t^{-1} (\lambda \theta_0 + \sum_{s=1}^{t-1} M_{a_s}^\T y_s)$ with inverse regularized covariance $V_t = \lambda I + \sum_{s=1}^{t-1} M_{a_s}^\T M_{a_s}$. In general form, $\hat \theta_t$ is the projection of the unconstrained least-square estimate onto $\Theta$ with respect to the $\|\cdot\|_{V_t}$ norm.

In settings where $\Theta$ is contained in a lower-dimensional affine subspace of $\RR^d$, or the dimension of the observation subspace $\opspan{M_c^\T}{c \in \aA}$ is significantly smaller than $d$, it can be useful to introduce a basis $W \in \RR^{d \times r}$ to ease computation and improve the dependence of the regret bounds on the dimension $d$. To be a valid basis, $W$ needs to satisfy $W^\T W = \eye_r$ and 
\begin{align}
	\forall a \in \aA,\,\, \theta, \nu \in \Theta\, \quad M_a(\theta - \nu) = M_a W W^\T (\theta - \nu)\,.\label{eq:W-def}
\end{align}
We remark that $W \in \RR^{d \times r}$ can be chosen to satisfy
\begin{align}
	r \leq \min\{\dim(\Theta),\, \dim(\im(M_c^\T) : c \in \aA)\} \leq \min\{d,\,m |\aA|\}\,.\label{eq:pm-r-def}
\end{align}
$W = \eye_d$ is a perfectly valid (and sometimes the only possible) choice.
The next lemma provides an elliptical confidence set for $\hat \theta_t$.
\begin{lemma}\label{lem:confidence} Let $\beta_{t,\delta}^{1/2} \eqdef \rho\sqrt{2 \log \tfrac{1}{\delta} + \log \det(W_t) - \log \det(\lambda \eye_r)} + \sqrt{\lambda}B$ be a confidence coefficient where $W_t = W^\T V_t W \in \RR^{r \times r}$.
  Then \[\PP\big[\forall t \geq 1,\, \theta^* \in \eE_{t, \delta} \eqdef \{\theta \in \Theta : \|\theta - \htheta_t\|_{V_t}^2 \leq \beta_{t,\delta} \}\big] \geq 1 - \delta\,.\]
\end{lemma}
The proof generalizes the standard ellipsoidal confidence set \cite[cf.][]{abbasi2011improved} and is deferred to \cref{proof:lem-confidence}.

In the following, we set $\beta_t = \beta_{t,1/t^2}$ and $\eE_t = \eE_{t,1/t^2}$, which allows us to derive bounds on the expected regret. It is also possible to fix the confidence level $\delta$ to obtain high-probability bounds, or tune the confidence coefficient empirically.

%
\paragraph{Gap Estimates} The gap estimates are defined as conservative estimates of the true gaps:
\begin{align}
	\hat \Delta_t(a) = \max_{\theta \in \eE_t,\, b \in \aA} \ip{\phi_b - \phi_a, \theta}\,.\label{eq:gap-def}
\end{align}
When $\Theta = \RR^d$, the maximum over $\theta$ can be computed in closed-form: $\hat \Delta_t(a) = \max_{b \in \aA} \ip{\phi_b - \phi_a, \hat \theta_t} + \beta_{t}^{1/2} \|\phi_b - \phi_a\|_{V_t^{-1}}$. Note that $\hat \Delta_t(a)$ is chosen as a high-probability upper bound on the true gap $\Delta(a)$, which allows us to control the error term in \cref{eq:ids-2}. Specifically, denoting $\Delta_{\max} = \max_{a \in \aA} \Delta(a)$, we get
\begin{align}
	\sum_{t=1}^n \EE[\Delta(a_t) - \hat \Delta_t(a_t)] \leq \sum_{t=1}^n \Delta(a_t) \PP[\theta^* \notin \eE_{t,1/t^2}] \leq \sum_{t=1}^n \Delta_{\max} t^{-2} \leq \oO(\Delta_{\max})\,. \label{eq:error}
\end{align}

\paragraph{Information Gain} The next step is to choose an information gain function $I_t(a)$.  We define the information gain as the increase of the log-determinant, given by 
\begin{align}
	I_t(a) &=  \sdfrac{1}{2} \log \det (W_t + W^\top M_a^\top M_a W)- \sdfrac{1}{2} \log \det(W_t) \nonumber\\
	&= \sdfrac{1}{2} \log \det\big(\eye_m + (M_a W) W_t^{-1} (M_a W)^\T\big)\,.\label{eq:info}
\end{align}
The log-determinant of the covariance matrix is a common progress measure in linear experimental design that captures the log-volume of the confidence ellipsoid (D-optimal design).
This choice was primarily analyzed in the frequentist IDS framework by \citet{kirschner2020pm}, with the difference that here we introduced the basis $W$. The definition further has a natural interpretation in the Bayesian setting \citep{russo2014ids}. Assume for a moment that inference is done with Gaussian prior $\vartheta \sim \nN(W^\T \theta_0, \lambda^{-1} \eye_r)$ and an observation likelihood $y_t \sim \nN(M_{a_t}(W\vartheta + (\eye_d - WW^\T)\theta_0), \eye_m)$. The posterior distribution corresponds to the least-squares estimate, $\nN(W^\T \hat \theta_t, W_t^{-1})$ and the entropy of the Gaussian posterior distribution on the subspace defined by $W$ is $\HH(\theta) = \frac{1}{2} \log( (2\pi e)^d \det(W_t^{-1}))$. Therefore \cref{eq:info} corresponds to the entropy reduction when choosing $a_t = a$ and observing $y_t = M_a \theta + \epsilon_t$, 
known as the \emph{mutual information}, 
\begin{align*}
	\II_t(\theta; y_t|a_t=a) = \sdfrac{1}{2} \log \det (W_{t+1})- \sdfrac{1}{2} \log \det(W_t) = I_t(a)\,.
\end{align*}
In light of \cref{thm:ids-regret-general},
an important quantity in our analysis is the total information gain $\gamma_n = \sum_{t=1}^n I_t(a_t)$. The next lemma is a standard result closely related to the \emph{elliptical potential lemma}. It provides a worst-case bound on $\gamma_n$ that is independent of the sequence of actions. 
\begin{lemma}[Total Information Gain]\label{lem:total-information}
	The total information gain $\gamma_n = \sum_{t=1}^n I_t(a_t)$ is bounded as follows,
	\begin{align*}
		\gamma_n = \frac{1}{2} \log \det (W_n) - \frac{1}{2}  \log \det (\lambda \eye_r) \leq \frac{r}{2} \log\left(1 + \frac{nL}{\lambda r}\right)\,.
	\end{align*}
\end{lemma}
For a proof, see, e.g.,~\citep[Lemma 19.4]{lattimore2019bandit}.
The lemma motivates the use of a lower-dimensional basis $W$ (where possible) as it makes the upper bound independent of the ambient dimension $d$. Further note that the lemma implies an upper bound on the confidence coefficient since $\beta_{t,\delta}^{1/2} \leq \beta_{n,\delta}^{1/2} = \rho(\gamma_n + 2\log \frac{1}{\delta})^{1/2} + \lambda^{1/2} B$.



\paragraph{Computation} Note that the least-squares estimate and inverse of the covariance matrix can be computed incrementally in $\oO(d^2)$ steps per round. Computing gap estimates requires $\oO(|\aA|^2 d^2)$ operations and computing the information gain can be done in $\oO(|\aA|d^2)$. Computing the IDS distribution requires $\oO(|\aA|^2)$, and an approximate distribution that minimizes the information ratio up to constant factor can be computed in linear time (see \cref{eq:approximate-IDS}). In \cref{ss:fast} we show how to improve the per-step computation complexity to $\oO(|\aA|d^2)$ by introducing a gap estimate that can be computed in $\oO(|\aA|d^2)$. By carefully computing quantities using the basis $W$, the computation complexity can be reduced to $\oO(|\aA| r^2)$ per round and $\oO(d)$ once at the beginning of the game.

\section{Regret Bounds}\label{sec:proofs}

How fast the learner can determine an optimal action in a linear partial monitoring game depends on the geometric structure of feedback and reward. Some terminology is necessary to state the results. The main distinction is between \emph{locally} and \emph{globally} observable games. Informally, in globally observable games, the learner has access to actions with which it can estimate $\ip{\phi_a - \phi_b, \theta^*}$ for all actions $a$ and $b$ that are Pareto optimal (defined formally below). The Pareto optimal actions have the property
that, for any $\theta^*$, one of them is always optimal.
However, acquiring sufficient information might incur a constant regret cost per round, and appropriately trading off exploration and exploitation leads to $\oO(n^{2/3})$ regret in the worst-case. In locally observable games, the cost payed for information is at most proportional to the statistical estimation error of the optimal action, in which case the learner can achieve $\oO(\sqrt{n})$ regret. Note that local observability is a stronger requirement than global observability. Any locally observable game is also globally observable.

\subsection{Local and Global Observability}

Following the standard terminology \citep{bartok2014partial}, the cell of $a \in \aA$ is defined as the set of parameters for which $a \in \aA$ is optimal,
\begin{align}
	\cC_a \eqdef \{\theta \in \Theta : \ip{\phi_a, \theta} = \max_{b \in \aA} \ip{\phi_b, \theta} \}\,.\label{eq:pm-cell-constrained}
\end{align}
Two actions $a, b \in \aA$ are called \emph{duplicates} if $\ip{\phi_a, \theta} = \ip{\phi_b, \theta}$ for all $\theta \in \Theta$. Note that duplicate actions can still differ on the feedback maps $M_a$ and $M_b$.
An action is called \emph{Pareto optimal} if $\dim(\cC_a) = \dim(\Theta)$. 
If $a$ is Pareto optimal, then the only actions that are optimal on the relative interior of $\cC_a$ are $a$ and its duplicates (which are also Pareto optimal).
The set of all Pareto optimal actions is $\pP$.  An action is called \emph{degenerate} if $0 \leq \dim(\cC_a) < \dim(\Theta)$, and \emph{dominated} if $\cC_a = \emptyset$.  Degenerate actions can be optimal but not uniquely so, whereas dominated actions are never optimal. We denote the linear span of parameter differences by $\vV = \laspan(\{\theta - \nu :  \theta,\nu \in \Theta\})$ and introduce the orthogonal projection $\PM : \RR^d \rightarrow \vV$. 

\paragraph{Global Observability}
A linear partial monitoring game is called \emph{globally observable} if 
\begin{align}
	\forall\, a,b \in \pP \,,\quad \PM (\phi_a - \phi_b) \in \opspan{\PM M_c^\T}{c \in \aA} \,. \label{eq:global-def}
\end{align}
Intuitively, the requirement is that the learner can estimate the difference in reward $\ip{\phi_a - \phi_b, \nu - \theta}$ for Pareto optimal actions $a,b \in \pP$ and parameters $\nu, \theta \in \Theta$ by combining the feedback from all actions. The projection onto $\vV$ appears naturally as the set of possible directions in which two parameters can differ. The worst-case {cost-to-signal} ratio is captured by the \emph{global alignment constant}:
\begin{align}
	\alpha \eqdef \max_{\nu \in \vV} \max_{a,b \in \pP} \min_{c \in \aA} \frac{\ip{\phi_a - \phi_b,\nu}^2}{\|M_c \nu\|^2}\,.\label{eq:global-alignment}
\end{align}
An immediate consequence is that $\alpha < \infty$ if and only if the condition in \cref{eq:global-def} is satisfied. 

\paragraph{Local Observability} The definition of \emph{local observability} strengthens the previous definition by requiring that the learner can estimate the reward difference 
between Pareto optimal actions that are plausibly optimal and that it can do so by playing actions with small regret.
Formally, given a set $\eE \subset \Theta$, let $\pP(\eE) =\cup_{\theta \in \eE} \{ a \in \pP :  \max_{b \in \aA} \ip{\phi_b - \phi_a, \theta} = 0\}$ 
be the set of Pareto optimal actions that are optimal for some $\nu \in \eE$. 
Later we will take $\eE$ to be the confidence set constructed by the algorithm in a given round. In this case $\pP(\eE)$ is the set of plausibly optimal actions in $\pP$.
Denote by  $\Delta(a|\theta) = \max_{b\in\aA} \ip{\phi_b - \phi_a, \theta}$ the gap of action $a \in \aA$ for parameter $\theta \in \Theta$. For any $\eta > 0$, the \emph{extended plausible Pareto} set is defined as follows:
\begin{align*}
	\bar \pP_\eta(\eE) = \{a \in \aA : \max_{\theta \in \eE} \Delta(a|\theta) \leq  \eta \cdot \max_{b \in \pP(\eE)} \max_{\theta \in \eE} \Delta(b|\theta) \} \,.
\end{align*}\todoj{new definition}
Actions in $\bar \pP(\eE)$ may be dominated or degenerate in general. The point is that for any action $a \in \bar \pP(\eE)$, there exists a Pareto optimal action $ b \in \pP(\eE)$ with larger regret under some plausible parameter (up to a constant factor). 
%
The local alignment constant for $\eE$ is
\begin{align}
	\alpha_\eta(\eE) \eqdef \max_{\nu \in \vV} \max_{a,b \in \pP(\eE)} \min_{c \in \bar \pP_\eta(\eE)} \frac{\ip{\phi_a - \phi_b,\nu}^2}{\|M_c \nu\|^2}\,.\label{eq:local-alignment}
\end{align}
We let $\alpha(\eE) \eqdef \inf_{\eta > 0} \eta^2 \alpha_\eta(\eE)$. A game is called \emph{locally observable}\todoj{new definition} if  
\begin{align}
	\sup_{\eE \subset \Theta} \alpha(\eE) < \infty\,.\label{eq:local-def}
\end{align}
The next lemma is helpful to to bound the alignment constant 
\citep[c.f.~Lemma 13]{kirschner2020pm}.
\begin{lemma} \label{lem:alignment-bounds}
	Let $a,b\in \aA$ with $\phi_a \neq \phi_b$, and $\bB\subset \aA$ be a subset of actions such that 
	\begin{align*}
		\quad \PM (\phi_a - \phi_b) \in \opspan{\PM M_c^\T}{c \in \bB}\,.
	\end{align*}
	Then there exist weights $w_{ab}^c \in \RR^m$ for each $c \in \bB$, such that $\ip{\phi_a - \phi_b, \nu} = \sum_{c \in \bB} \ip{ M_c^\T w_{ab}^c, \nu}$ for all $\nu \in \vV$. Further, for any such weights,
	\begin{align*}
		\max_{\nu \in \vV} \min_{c \in \bB} \frac{\ip{\phi_a - \phi_b, v}^2}{\|M_c^\T v\|^2} \leq \bigg(\sum_{c \in \bB} \|w_{ab}^c\|\bigg)^2 \,.
	\end{align*}

\end{lemma}

\begin{proof}
	The existence of the weights $w_{ab}^c \in \RR^m$ is immediate by assumption. Therefore, we can write
	\begin{align*}
		\ip{\phi_a - \phi_b, \nu}^2 = \ip{\textstyle \sum_{c \in \bB} M_c^\T w_{ab}^c, \nu}^2 = \Big(\sum_{c \in \bB}\ip{w_{ab}^c, M_c \nu}\Big)^2\,,
	\end{align*}
	An application of the Cauchy-Schwarz inequality shows:
	\begin{align*}
		\frac{\ip{\phi_a - \phi_b, \nu}^2}{\max_{c \in \bB} \|M_c \nu\|^2} 
		\leq \frac{\left(\sum_{c \in \bB} \|w_{ab}^c\| \|M_c \nu\|\right)^2}{\max_{c \in \bB}\|M_c \nu\|^2} \leq \bigg(\sum_{c \in \bB} \|w_{ab}^c\|\bigg)^2\,.
	\end{align*}
\end{proof}
The lemma can be used to bound the alignment constant for various games that were introduced as examples. We refer to \cref{tab:examples} for an overview.
\begin{table*}[t]
	\scriptsize
	\renewcommand{\arraystretch}{1.2}
	\begin{center}
		\begin{tabular}{|p{4.0cm}p{2.2cm}p{3.cm}p{4.4cm}|}
			\hline
			\textbf{Example} & \textbf{Local / Global} & \textbf{Alignment Constant} & \textbf{IDS Regret} $\R_n$ \\ \hline
			Linear Bandit (Ex.~\ref{ex:lb})        & Local  &$\alpha(\eE_t) \leq 4$ &  $\oO(\sqrt{n} d \log (n))$  \\ 	
			Dueling Bandit (Ex.~\ref{ex:dueling})        & Local  &$\alpha(\eE_t) \leq 4$ &  $\oO(\sqrt{n} d \log (n))$  \\ 
			Graph Dueling Bandit (Ex.~\ref{ex:graph-dueling})        & Global  &$\alpha(\eE_t) \leq 4 W_{\max}^2$ &  $\oO((n d \log (n)W_{\max})^{2/3})$  \\ 
			Non-Degenerate Finite PM        & Local  &$\alpha(\eE_t) \leq 4k^2m^2$ &  $\oO\left(k m^{3/2}n^{1/2}\,r \log(nr)\right)$  \\ 
			Finite PM        & Local  &$\alpha(\eE_t) \leq  4mdk^{d+2}$ & $\oO((n m d k^{d+2})^{1/2} r \log(rn)$  \\ 
			Finite PM       & Global  &$\alpha \leq mdk^{d+2}$ &   $ \oO\left( (mdk^{d+2})^{1/3} (r  n \log(rn))^{2/3}\right)$  \\ 
			\hline	
		\end{tabular} 
	\end{center}
	
	\caption{
	The regret bound follows from \cref{thm:regret-global,thm:regret-local}, and $\beta_n, \gamma_n \leq \oO(r \log (rn))$. \looseness=-1
	}\label{tab:examples}
\end{table*}

\paragraph{Linear Bandits} are locally observable. This follows from the observation that for any two actions $a,b$, we have $\phi_a - \phi_{b} \in \laspan\{\phi_a, \phi_b\}$. We can choose the estimation vector $w_{a,b}^a = -w_{a,b}^b = 1$ and $w_{a,b}^c = 0$ for all $c \notin \{a,b\}$. Therefore $\alpha(\eE_t) \leq \alpha_1(\eE_t) \leq \max_{a,b\in \aA} (|w_{a,b}^a| + |w_{a,b}^b|)^2 = 4$.
\paragraph{Dueling Bandits} are also locally observable. Recall that we defined dueling bandits on a ground set $\iI$ with features $\phi_a$ for all $a \in \iI$ by letting $\aA = \iI \times \iI$, $\phi_{a,b} = \phi_a + \phi_b$ and $M_{a,b} = \phi_a - \phi_b$. Note that dueling actions $(a,b)$ with $\phi_a \neq \phi_b$ cannot be Pareto optimal. Let $(a,a), (b,b) \in \aA$ be a pair of Pareto optimal dueling actions. A simple calculation reveals that $\phi_{a,b} \in \conv(\{\phi_{a,a,}, \phi_{b,b}\})$, and $\phi_{a,a} - \phi_{b,b} = 2 M_{a,b}$. It follows from \cref{lem:alignment-bounds} that $\alpha(\eE_t) \leq \alpha_1(\eE_t)\leq 4$.
\newcommand{\dist}{\text{dist}}
\paragraph{Graph Dueling Bandits} use the same reward features and feedback maps as dueling bandits, but restrict the action set to a subset $\aA \subset \iI \times \iI$. Define $\dist(a,b)$ as the shortest undirected path from $a$ to $b$ in the graph defined by $(\iI,\aA)$, or infinity if no such path exists. Let $W_{\max} = \max_{a,b \in \pP} \dist(a,b)$ be the maximum length of a shortest path between any two Pareto optimal actions. This game is globally observable if and only if $W_{\max} < \infty$. 

To see this, note that for any pair of Pareto optimal dueling actions $(a,a)$, $(b,b)$ we can find a sequence $c_0, c_1, \dots, c_l$ with $c_0 = a$ and $c_l = b$ and $l \leq W_{\max}$, such that $(c_i, c_{i+1}) \in \aA$ for all $i=0, \dots, l-1$. In particular, we can write $\phi_{a,a} - \phi_{b,b} = 2 \sum_{i=1}^{l-1} M_{c_i,c_{i+1}}$ and \cref{lem:alignment-bounds} implies a bound on the alignment constant, $\alpha \leq 2 W_{\max}$.
	
\paragraph{Finite Partial Monitoring}
Bounds for finite partial monitoring games are derived in the next lemma. 
A finite partial monitoring game is called non-degenerate if every action is either Pareto optimal or dominated
and there are no duplicate actions.

\newcommand{\Saspan}{\oplus_{c \in \aA} \Im (S_c^\T)}

\begin{lemma}\label{lem:alpha-global}
	Let $\eE \subset \Theta$ be convex.
	For finite partial monitoring games the following bounds on the alignment constant hold:
	\begin{enumerate}
		\item[(a)] For globally observable games $\alpha \leq m dk^{d+2}$; and
		\item[(b)] for locally observable games $\alpha(\eE) \leq 4mdk^{d+2}$; and
		\item[(c)] for non-degenerate locally observable games $\alpha(\eE) \leq 4 k^2 m^3$.
	\end{enumerate}
\end{lemma}
The proof is given in \cref{proof:alpha-global}.

\subsection{Regret Bounds}

Using the notion of local and global observability, we can now state the main results. To interpret the results, note that $\gamma_n, \beta_n \leq \oO(r \log(1 + n))$ by \cref{lem:total-information}. For an overview on how the regret bounds behave on standard examples, see \cref{tab:examples}.
\begin{theorem}[Globally Observable Games]\label{thm:regret-global} Let $\gG$ be a globally observable game that satiesfies \cref{ass:bounded} with worst-case alignment $\alpha$ according to \cref{eq:global-alignment}. Then the regret of IDS (\cref{alg:ids-pm})  with $\lambda \geq L$ satisfies 
	\begin{align*}
		\R_n \leq n^{2/3} \left(54\alpha  \EE[\beta_n]\EE[\gamma_n (\Delta_{\max} + 4B)]\right)^{1/3} + \oO(\Delta_{\max})\,.
	\end{align*}
Consequently, $\R_n \leq \oO( \alpha^{1/3} r \log(1+n) n^{2/3})$.
\end{theorem}
\begin{proof}
	Combining the general IDS regret bound in \cref{thm:ids-regret-general} with $\kappa=3$ and using \cref{eq:error} to bound the estimation error, we find
	\begin{align*}
		\R_n \leq n^{2/3}\EE[\gamma_n]^{1/3} \left(\frac{1}{n}\EE[\sum_{t=1}^n \Psi_{3,t}] \right)^{1/3} + \oO(\Delta_{\max})\,.
	\end{align*}
	The next step is to bound the information ratio. The main idea is to optimize the trade-off between playing a greedy action $\hat a_t =  \argmax_{a \in \aA} \ip{\phi_a, \hat \theta_t}$ and the action that maximizes information gain. As it turns out, this is sufficient to bound the information ratio in globally observable games. We emphasize that this particular choice of actions appears only in the analysis. The actual IDS distribution may be supported on actions that achieve an even smaller information ratio.
	
	Note that we may always choose $\hat a_t \in \pP$ as a Pareto optimal action. Hence
	\begin{align*}
		\hDelta_t(\hat a_t) =  \max_{\theta \in \eE_t} \max_{b \in \aA} \ip{\phi_b - \phi_{\hat a_t}, \theta} =  \max_{\theta \in \eE_t} \max_{b \in \pP} \ip{\phi_b - \phi_{\hat a_t}, \theta} \leq \max_{\theta \in \eE_t} \max_{b \in \pP} \ip{\phi_b - \phi_{\hat a_t}, \theta - \htheta_t}\,.
	\end{align*}
	The equality uses that we can choose the maximizer in $\pP$. Using the definition of the global alignment constant $\alpha$ in \cref{eq:global-alignment}, we find
	\begin{align}
		\hat \Delta_t(\hat a_t)^2 &\leq \max_{\theta \in \eE_t} \max_{b \in \pP} \ip{\phi_b - \phi_{\hat a_t}, \theta - \htheta_t}^2 \leq \alpha \max_{\theta \in \eE_t} \max_{c \in \aA} \|M_c (\theta - \htheta_t) \|^2 \,.\label{eq:global-1}
	\end{align}
	Using \cref{eq:W-def} and Cauchy-Schwarz, we can further bound for all $c \in \aA$,
	\begin{align}
		\max_{\theta \in \eE_t}	\|M_c (\theta - \htheta_t) \|^2  &= \max_{\theta \in \eE_t}\|M_c W W^\T (\theta - \htheta_t) \|^2\nonumber\\
		&\leq \max_{\theta \in \eE_t} \|W^\T (\theta - \hat \theta_t)\|_{W_t}^2 \|M_c W \|_{W_t^{-1}}^2  \leq\beta_t \|M_c W \|_{W_t^{-1}}^2 \label{eq:I-bound}
	\end{align}
	The last inequality follows from the definition of the confidence set $\eE_t$. Moreover, note that $\|M_c W \|_{W_t^{-1}}^2 \leq \|M_c W \|_{W_0^{-1}}^2 \leq L \lambda^{-1} \leq 1$ by our assumption on $\lambda$. Hence, using further that $x \leq 2 \log(1 + x)$ for $x \in [0,1]$,
	\begin{align}
	 \|M_c W \|_{W_t^{-1}}^2 &= \lambda_{\max} (M_c W W_t^{-1} (M_c W)^\T) \nonumber\\
		&\leq 2\log(1 + \lambda_{\max} (M_c W W_t^{-1} (M_c W)^\T)) 
		\nonumber\\
		&\leq 2 \log \det (\eye_m + M_c W  W_t^{-1} (M_cW)^\T) = 4 I_t(c) \label{eq:norm-logdet-bound} 
	\end{align}
	Taking the previous two displays together, we get $\max_{\theta \in \eE_t}	\|M_c (\theta - \htheta_t) \|^2  \leq 4 \beta_t I_t(c)$. 
Combined with \cref{eq:global-1}, we get
\begin{align*}
	\hat \Delta_t(\hat a_t)^2 &\leq 4 \alpha \beta_t \max_{c \in \aA} I_t(c)\,.
\end{align*}
To bound the generalized information ratio $\Psi_{3,t}$  let $c_t = \argmax_{a \in \aA} I_t(a)$. Using \cref{lem:psi-kappa-bound}, we find
\begin{align}
\nonumber	\Psi_{3,t}(\mu_t) \leq 2 \min_{\mu \in \sP(\aA)} \Psi_{3,t}(\mu) &\stackrel{(i)}{\leq} 2 \min_{p \in [0,1]} \frac{\big((1-p) \hat \Delta_t(\hat a_t) + p \hat \Delta_t(c_t)\big)^3}{p I_t(c_t)}\\
\nonumber	&\stackrel{(ii)}{\leq} \frac{27 \hat \Delta_t(\hat a_t)^2 \hat \Delta_t(c_t)}{2 I_t(c_t)}\\
	&\stackrel{(iii)}{\leq} 54  \alpha \beta_t( \Delta_{\max}+ 4B) \label{eq:Psi_3 bound}
\end{align}
Step $(i)$ uses that $I_t(\hat a_t) \geq 0$, $(ii)$ follows with $p= \frac{\hat \Delta_t(\hat a_t)}{2 \hat \Delta_t(c_t)}$ and $(iii)$ the inequality in the previous display and a direct consequence of the boundedness \cref{ass:bounded}, $\hat \Delta_t(c_t) \leq \Delta_{\max} + \max_{a,b} \max_{\theta \in \eE_t} \ip{\theta - \theta^*, \phi_a - \phi_b} \leq \Delta_{\max} + 4B$ . 

The final bound follows using \cref{lem:total-information} to bound $\beta_n$ and $\gamma_n$.
\end{proof}

The locally observable case is summarized in the next theorem.
\begin{theorem}\label{thm:regret-local}
	 Let $\gG$ be a locally observable partial monitoring game that satisfies \cref{ass:bounded}. Then the regret of IDS (\cref{alg:ids-pm}) with regularizer $\lambda \geq L$ satisfies
	\begin{align*}
		\R_n \leq \sqrt{8 \EE[ \bar \alpha_n\beta_n ] \EE[\gamma_n] n} + \oO(\Delta_{\max})\,,
	\end{align*}
	where $\bar \alpha_n = \frac{1}{n}\sum_{t=1}^n\alpha(\eE_t)$ is the average realized local alignment constant. In particular, $\R_n \leq \oO(r \sqrt{n \EE[\bar \alpha_n]} \log(1 +n))$.
\end{theorem}
\begin{proof}
	We start once more with the general IDS regret bound in \cref{thm:ids-regret-general} for $\kappa=2$ and use \cref{eq:error} to bound the estimation error, which gives
	\begin{align*}
		\R_n \leq \sqrt{\EE[\gamma_n]n} \sqrt{\frac{1}{n}\sum_{t=1}^n \Psi_{2,t}(\mu_t)} + \oO(\Delta_{\max})\,.
	\end{align*}
	To bound the information ratio, we make use of the plausible Pareto action set $\pP(\eE_t)$ and its convex relaxation $\bar \pP(\eE_t)$.
	 For any $\eta > 0$, let $c_t = c_t(\eta) = \argmax_{a \in \bar \pP_\eta(\eE_t)} I_t(a)$ be the most informative action in $\bar \pP_\eta(\eE_t)$. 
By the definition of $\bar \pP_\eta(\eE_t)$, we bound the gap estimate 
as follows,
	\begin{align}
	\eta^{-1}	\hat \Delta_t(c_t) \leq \max_{a,b \in \pP(\eE_t)} \max_{\theta \in \eE_t}\ip{\phi_a - \phi_b, \theta} \leq \max_{a,b \in \pP(\eE_t)} \max_{\theta, \nu \in \eE_t}  \ip{\phi_a - \phi_b, \theta - \nu}\,.\label{eq:local-1}
	\end{align}
	The second inequality follows since $\nu \in \eE_t$ can be chosen such that $b$ is optimal for $\nu$. Consequently the definition of the local alignment constant in \cref{eq:local-alignment} implies
	\begin{align*}
		\hat \Delta_t(c_t)^2 &\leq \eta^2 \alpha_\eta(\eE_t) \max_{\nu, \omega \in \eE_t } \max_{c  \in \bar \pP_\eta(\eE_t)}  \|M_{c} (\nu - \omega)\|^2 
	\end{align*}
	To relate the norm to the information gain, note that for any $c \in \aA$,
	\begin{align*}
	   \max_{a,b \in \pP(\eE_t)}\|M_{c} (\nu - \omega)\|^2&=   \max_{a,b \in \pP(\eE_t)}\|M_{c}W W^\T (\nu - \omega)\|^2\\
	&\leq  \max_{a,b \in \pP(\eE_t)} \| W^\T(\nu - \omega)\|_{W_t}^2\|M_{c}W\|_{W_t^{-1}}^2  \leq 2   \beta_t \|M_{c} W\|_{W_t^{-1}}^2 \,.
		\end{align*}
	From the argument in \cref{eq:norm-logdet-bound}, it follows that $\|M_{c} W\|_{W_t^{-1}}^2 \leq 4 I_t(c)$. Therefore we bound the information ratio for the action $c_t$ as follows,
	\begin{align*}
		\hat \Delta_t(c_t)^2&\leq \max_{c \in \bar{\pP}_\eta}2  \eta^2  \alpha_\eta(\eE_t)\beta_t  \|M_{c} W\|_{W_t^{-1}}^2 \leq 8 \eta^2 \alpha_\eta(\eE_t) \beta_t I_t(c_t)\,.
	\end{align*}
	This shows that to bound the information ratio it suffices to play $c_t$,
	\begin{align*}
		\Psi_t(\mu_t) = \min_{\mu \in \sP(\aA)} \Psi_t(\mu) \leq \inf_{\eta > 0} \min_{c \in \bar \pP_\eta(\eE_t)} \frac{\hat \Delta_t(c)^2}{I_t(c)} \leq \inf_{\eta > 0} 8 \eta^2 \alpha_\eta(\eE_t)\beta_t  =  8 \alpha(\eE_t)\beta_t\,,
	\end{align*}
	and the regret bound follows because $(\beta_t)$ is increasing, and using \cref{lem:total-information} to bound $\beta_n$ and $\gamma_n$.
\end{proof}


\subsection{Classification of Finite Action Games}

For the upper bounds, so far we have encountered two cases: globally and locally observable games that satisfy the conditions in \cref{eq:global-def,eq:local-def} respectively. 
There are two other types of games with a less interesting structure. A game is called \emph{trivial} if there exists an action $a$ such 
that $\cC_a = \Theta$. By definition, $a$ is optimal for all possible values of $\theta_*$. IDS achieves zero regret on these games because $\hat \Delta_t(a) = 0$ and the information ratio is minimized by a Dirac on $a$.

Games that are locally observable but not trivial are called \emph{easy}. Games that are globally observable but not locally observable are called \emph{hard}.
The last category of games are those that are not globally observable, which are called \emph{hopeless}.
The reason is that in these games there are necessarily multiple (non-duplicate) Pareto optimal actions that cannot be distinguished, which means the learner suffers
linear regret in the worst case.

The \emph{minimax regret} is $\R_n^* = \inf_{\pi} \sup_{\theta \in \Theta} \EE[R_n(\pi,\theta)]$, where the infimum is over all possible policies $\pi$ that map observation histories to actions. The classification theorem provides the minimax regret rate for each type of game.

\begin{theorem}[Classification]
	For any finite linear partial monitoring game, the minimax regret $\R_n^*$ satisfies
	\begin{align*}
		\R_n^* = 
		\begin{cases}
			0 & \text{for trival games,} \\
			\tilde\Theta(n^{1/2}) & \text{for easy games,} \\
			\tilde\Theta(n^{2/3}) & \text{for hard games,} \\
			\Omega(n) & \text{otherwise, for hopeless games}\,.
		\end{cases}
	\end{align*}
\end{theorem}
The classification theorem recovers the known result for finite stochastic partial monitoring up to logarithmic factors \citep{bartok2014partial,lattimore2019information}. It further generalizes the results for linear partial monitoring without explicit constraints on the parameter sets by \citet{kirschner2020pm}.

\begin{proof}
	As we have argued, IDS provides the upper bounds in each case. Therefore it remains to provide lower bounds for each category. These follow essentially from the results in \citep[Appendix F]{kirschner2020pm}, by restricting all constructions to the subspace $\vV$.
\end{proof}

\looseness -1 We conclude the section remarking that the situation in continuous action games is much more delicate. As shown in \citep[Section 2.4]{kirschner2020pm} the achievable rates are not only determined by the observability conditions, but further depend on the curvature of the action set. A full classification of continuous action games is still a fascinating open question.

\section{Extensions}\label{sec:extensions}

\subsection{Faster Gap Estimation}\label{ss:fast}

The version of IDS presented in \cref{alg:ids-pm} requires $\oO(|\aA|^2)$ computation steps per round in general, which quickly becomes prohibitive for large action sets. The main bottleneck is the computation of the gap estimates as defined in \cref{eq:gap-def}. We propose a novel gap estimate that can be computed in linear time. To this end, let $\hat a_t = \argmax_{a \in \pP} \ip{\hat \theta, a}$ be the empirically best action, chosen as a Pareto optimal action if not unique. We relax the gap estimate using $\hat a_t$ as an intermediate action:
\begin{align}
	\hat \Delta_t(a;\hat a_t) \eqdef \delta_t + \max_{\theta \in \eE_t} \ip{\phi_{\hat a_t} - \phi_a, \theta},\quad \text{where}\quad \delta_t \eqdef \max_{\theta \in \eE_t} \max_{b \in \aA} \ip{\phi_b - \phi_{\hat a_t}, \theta}\,.\label{eq:gap-relaxed-def}
\end{align}
It is immediate that $\hat \Delta_t(a) \leq \Delta_t(a;\hat a_t)$ holds for all $a \in \aA$. Note that $\hat a_t$ can be found by enumerating the Pareto optimal actions, and the gap $\Delta_t(a; \hat a_t)$ estimate can be computed for all $a \in \aA$ by solving $2k$ second-order cone programs over $\Theta$ with positive semi-definite quadratic constraints. 
Combined with the $\frac{4}{3}$-approximation of the IDS distribution using \cref{eq:approximate-IDS}, this allows us to reduce the computational complexity of \cref{alg:ids-pm} to $\oO(|A|)$ per round. The next lemma confirms that the new algorithm satisfies essentially the same bounds as before.
\begin{lemma}
	\cref{alg:ids-pm} with the gap estimate $\hat \Delta_t(a; \hat a_t)$ and approximate IDS sampling according to \cref{eq:gap-relaxed-def} satisfies the same bounds up to constants as in \cref{thm:regret-local,thm:regret-global}.
\end{lemma}
\begin{proof}
	The proofs of \cref{thm:regret-global,thm:regret-local} go through almost unchanged. We only need to derive an upper bound on the new gap estimate, as we explain below.

	In the globally observable case, we show that the analog of \cref{eq:global-1} holds. By definition of $\hat \Delta_t(a;\hat a_t)$ we get
	\begin{align*}
		\hat \Delta_t(\hat a_t;\hat a_t)^2 &\leq \max_{\theta \in \eE_t} \max_{b \in \pP} \ip{\phi_b - \phi_{\hat a_t}, \theta - \htheta_t}^2
	\end{align*}
	The remaining proof remains unchanged.
	
	In the locally observable case, we show that \cref{eq:local-1} continues holds up to a factor of two. For all $c \in \pP(\eE_t)$, we get
	\begin{align*}
	\hat \Delta_t(c;\hat a_t) \leq 2 \max_{a,b \in \pP(\eE_t)} \max_{\theta \in \eE_t}\ip{\phi_a - \phi_b, \theta} 
	\end{align*}
	The remaining proof remains unchanged with an additional factor of two in the leading term.
\end{proof}
Lastly, we remark that when  $\Theta = \{\theta \in \RR^d : \|\theta\|_2 \leq B\}$, then we can also use the following estimate based on truncation: \todot{Tor: Can we say something about finite PM? Relaxing $\eE_t$ to an elipsoid on aff($\Theta$)? }
\todoj{I think this works but one needs carefully reason about the regret contribution from offset of the affine space. }
\begin{align}
	\tilde \Delta_t(a;\hat a_t) \eqdef \min\{\delta_t +  \ip{\phi_{\hat a_t} - \phi_a, \hat \theta_t}, B\},\quad \text{where}\quad \delta_t \eqdef \max_{\theta \in \eE_t} \max_{b \in \aA} \ip{\phi_b - \phi_{\hat a_t}, \theta}\,.\label{eq:gap-mean-def}
\end{align}
The main difference to $\hat \Delta_t(a;\hat a_t)$ is that the mean-gap $\ip{\phi_{\hat a_t} - \phi_a, \hat \theta_t}$ appears directly in the estimation. One easily confirms that $\hat \Delta_t(a, \hat a_t) \leq 2 \tilde \Delta_t(\hat a_t, a_t)$ and the proof can be adjusted to accommodate this change. The advantage of $\tilde \Delta_t(a;\hat a_t)$ is that it is less dependent on the tightness of the confidence bound thereby can be more accurate in practice.

\subsection{Directed Information Gain}

The information gain function \cref{eq:info} based on the log determinant potential is a convenient choice for the worst-case analysis, but can be conservative in practice. The main reason is that \cref{eq:info} is agnostic to the current parameter estimate and the geometric structure of the game. Therefore, IDS acquires information that leads to uncertainty reduction on the true parameter $\theta^*$ overall, independent of the true maximizer $a^*$.

To introduce a form of \emph{directed} information gain, we define for any $\omega \in \RR^d$,
\begin{align}
	J_t(a;\omega) \eqdef \frac{1}{2}\| M_a \omega\|^2\,. \label{eq:info-omega}
\end{align}
Note that $J_t(a;\omega)$ measures the sensitivity of the feedback obtained from action $a \in \aA$ along the direction $\omega$. Next we define optimistic and pessimistic parameters 
\begin{align*}
	\theta_t^+, \theta_t^{-} = \argmax_{\theta_1, \theta_2 \in \eE_t} \max_{a,b \in \pP(\eE_t)} \ip{\phi_{b} - \phi_{a}, \theta_1 - \theta_2}\,.
\end{align*}
Both parameters can be computed by solving $|\pP|^2$ second-order cone programs over $\eE_t$. The computation complexity can be reduced to $2|\pP|$ using a relaxation on $\hat a_t$, similarly to the previous section. We define the directed information gain
\begin{align*}
	J_t(a) = \beta_t^{-1} J_t(a; \theta_t^+ - \hat \theta_t^{-})\,.
\end{align*}

\begin{lemma}
	\cref{alg:ids-pm} with the information gain $J_t(a)$ instead of $I_t(a)$ satisfies the same bounds as in \cref{thm:regret-local,thm:regret-global}.
\end{lemma}
\begin{proof}
	First note that the new information gain is upper bounded by the information gain $I_t$ (\cref{eq:info}) as follows,
	\begin{align*}
		J_t(c) =  \beta_t^{-1} \|M_c (\theta_t^+ - \theta_t^{-})\|^2 \leq 4 I_t(c)\,.
	\end{align*}
	The inequality follows along the same lines as \cref{eq:I-bound}. In particular, the total information gain remains bounded, $\gamma_{n,J} = \sum_{t=1}^n J_t(a_t) \leq 4 \sum_{t=1}^n I_t(a_t) = 4 \gamma_{n,I}$.
	
	Therefore it remains to show that the information ratio is bounded. For the globally observable case, note that
	\begin{align*}
		\hat \Delta_t(\hat a_t)^2 \leq \max_{b \in \pP} \max_{\nu \in \eE_t} \ip{\phi_b - \phi_{\hat a_t}, \nu - \hat \theta_t}^2
		&\leq \max_{b \in \pP} \ip{\phi_b - \phi_{\hat a_t}, \theta^+ - \theta^{-}}^2\\
		&\leq \alpha \max_{c \in \pP} \|M_c (\theta^+ - \theta^-)\| = \alpha \beta_t \max_{c \in \aA} J_t(c)\,.
	\end{align*}
	This follows along the same lines as \cref{eq:global-1}, and the remaining proof in the globally observable case remains unchanged.
	
	In the locally observable case, we reproduce \cref{eq:local-1} as follows
	\begin{align*}
		\hat \Delta_t(c_t)^2 &\leq \max_{a, b \in \pP(\eE_t)} \ip{\phi_b - \phi_{\hat a}, \theta^+ - \theta^{-}}^2\\
		&\leq \alpha(\eE_t) \max_{c \in \bar \pP(\eE_t)} \|M_c (\theta^+ - \theta^-)\|^2 = \beta_t \alpha(\eE_t) \max_{c \in \aA^+(\eE_t)} J_t(c)\,.
	\end{align*}
	The remaining proof follows unchanged.
\end{proof}

\section{Contextual Partial Monitoring}\label{sec:contextual}

\newcommand{\zdistr}{\chi}

\looseness -1 In the contextual bandit problem, the environment provides a context in each round \emph{before} the learner chooses an action \citep{langford2008epochgreedy}. Depending on our assumptions, the context may be either sampled from a fixed distribution or chosen adversarially by the environment. The reward depends on the context and the chosen action, and the learner competes with the best context-dependent action. In applications, the context can represent additional information available to the learner, such as temperature measurements, time of day, or the profile of a user visiting a website. We now introduce a \emph{contextual version of linear partial monitoring}. This setting directly generalizes the linear contextual bandit model.

Let $\zZ$ be a compact context set and $z_t \in \zZ$ the context presented at \mbox{time $t$}. 
 Reward features $\phi_a^z \in \RR^d$ and feedback map $M_a^z \in \RR^{d\times m}$ depend on the context $z \in \zZ$. For each context $z \in \zZ$ a subset of actions $\aA(z) \subset \aA$ is available for playing. The reward function is parameterized by a single $\thetaopt \in \RR^d$ shared among all contexts. The feedback in round $t$ for action $a_t$ and context $z_t$ is $y_t = M_{a_t}^{z_t} \thetaopt + \epsilon_t$, where $\epsilon_t \in \RR^m$ is conditionally independent $\rho$-sub-Gaussian noise.  The regret is defined so that the learner competes with the best action $a^*(z) = \argmax_{a \in \aA(z)} f(a, z)$ chosen in hindsight for each context $z \in \zZ$. For a sequence of contexts $(z_t)_{t=1}^n$, the \emph{contextual regret} is
\begin{align*}
	\R_n(\pi,\theta^*, (z_t)_{t=1}^n) = \EE[\sum_{t=1}^n  \ip{\phi_{a^*(z_t)}^{z_t} - \phi_{a_t}^{z_t}, \theta^*}]\,.
\end{align*}

Over the next two sections, we develop IDS policies for the contextual partial monitoring setting. The first variant directly extends \cref{alg:ids} by conditioning the information ratio on the observed context $z_t$. We refer to this variant as \emph{conditional IDS}. The second variant assumes that the context is sampled from a fixed and known distribution. This allows the learner to optimize the information ratio directly over the joint action-context distribution, which we refer to \emph{contextual IDS}. We show that this allows for much weaker conditions where sublinear regret is possible.

For simplicity, our presentation focuses on the case with unconstrained parameter set, $\Theta = \RR^d$. As before, we assume boundedness of the observation and reward features $\|M_a^z\|_2 \leq L$ and $\diam(\phi(a,z) : a \in \aA, z \in \zZ) \leq 1$ and parameter $\|\theta^*\|_2 \leq B$. The case with constrained parameter set can be handled as explained in \cref{sec:ids} by introducing appropriate projections. The product action space over the contexts with some fixed ordering is $\aA(\zZ) \eqdef \times_{z \in \zZ} \aA(z)$.

\subsection{Conditional IDS}
The definitions of the gap estimate and the information gain directly extend to the contextual setting. As before $\hat \theta_t$ denotes the least squares estimate defined in \cref{eq:ls}, and $\eE_t = \eE_{t,1/t^2}$ the calibrated confidence set defined in \cref{lem:confidence}.
The conditional gap estimate for a context $z \in \zZ$ and action $a \in \aA(z)$ is 
\begin{align}
	\hDelta_t(a,z) &= \max_{b \in \aA(z)} \max_{\theta \in \eE_t} \ip{\phi_b^z-\phi^z_a,\theta}\,. \label{eq:gap-contextual}
\end{align}
The (undirected) information gain for action $a \in \aA$ and context $z \in \zZ$ is
\begin{align}
	I_t(a,z) &= \sdfrac{1}{2} \log \det(\eye + M_a^{z\T} V_{t}^{-1} M_a^z)\,. \label{eq:info-contextual}
\end{align}

\emph{Conditional IDS} is the policy that optimizes the information ratio conditioned on the observed context $z_t$ at time $t$,
\begin{align}
	\mu_t(z_t) = \argmin_{\mu \in \sP(\aA(z_t))} \left\{\Psi(\mu, z_t) = \frac{\hat \Delta_t(\mu, z_t)^2}{I_t(\mu, z_t)}\right\}\,, \label{eq:conditional-ids}
\end{align}
and samples the action according to $a_t \sim \mu_t(z_t)$.
%
%
%
The computational complexity required to find the minimizer of the information ratio is the same as in the non-contextual case. 

\paragraph{Regret Bounds} For the analysis, we extend the notion of the alignment constant with the contextual argument. We let $\pP(z)$ be the set of Pareto optimal actions for context $z \in \zZ$ defined as actions that are uniquely optimal for some parameter $\theta$ and context $z$. The conditional global alignment constant is
\begin{align}\label{eq:conditional-alignment}
	\alpha(z) \eqdef \max_{\omega \in \RR^d} \max_{a,b \in \pP(z)} \min_{c \in \aA(z)} \frac{\ip{\phi^z_a - \phi^z_b,\omega}^2}{\|M_c^z \omega\|^2}\,.
\end{align}
 Further, let $\Delta(a,z|\theta) = \max_{b\in \aA(z)} \ip{\phi_b^z - \phi_a^z, \theta}$ be the gap of action $a \in \aA(z)$ under parameter $\theta \in \Theta$ and context $z \in \zZ$ and define contextual extensions of the (extended) plausible Pareto optimal actions as follows:
\begin{align*}
	\pP(\eE, z) &\eqdef \cup_{\theta \in \eE} \{a \in \pP(z): \ip{\phi_a^z,\theta} = \max_{b \in \aA(z)}\ip{\phi_b^z,\theta}\}\,,\\
	\bar \pP_\eta(\eE, z) &\eqdef \{a \in \aA(z) : \Delta(a,z|\theta) \leq \eta \cdot \max_{b \in \pP(\eE,z)} \Delta(a,z|\theta)\}\,. 
\end{align*} 
The local alignment constant for $\eta > 0$ under context $z \in \zZ$ is then
\begin{align}\label{eq:conditional-alignment-local}
	\alpha_\eta(\eE, z) \eqdef \max_{\omega \in \RR^d} \max_{a,b \in \pP(\eE, z)} \min_{c \in \bar \pP_\eta(\eE, z)} \frac{\ip{\phi^z_a - \phi^z_b,\omega}^2}{\|M_c^z \omega\|^2}\,,
\end{align}
and we let $\alpha(\eE, z) \eqdef \inf_{\eta > 0} \eta^2 \alpha_\eta(\eE, z)$.

The next two results are immediate extensions of the upper bound  for globally and locally observable games that replace the bounds on the information ratio with their conditional counterpart. These bounds are only meaningful if the game is globally or locally observable for each observed context $z_t \in \zZ$.
\begin{corollary}\label{cor:global-upper-conditional}
	For any $\fF_t$-predictable sequence $(z_t)_{t=1}^n$ in $\zZ$, the regret of conditional IDS satisfies,
	\begin{align*}
		\R_n \leq \oO \left(n^{2/3}(\EE[\bar \alpha_n\beta_n]\EE[\gamma_n])^{1/3}\right)\,,
	\end{align*}
	where $\bar \alpha_n = \frac{1}{n} \sum_{t=1}^n \alpha(z_t)$ is the average global alignment on the observed sequence of contexts.
\end{corollary}
\begin{proof}
	The claim follows along the same lines of as \cref{thm:regret-global}.
\end{proof}
The result for the locally observable case is stated similarly.
\begin{corollary}\label{cor:local-upper-conditional}
	For any $\fF_t$-predictable sequence $(z_t)_{t=1}^n$ in $\zZ$, the regret of conditional IDS satisfies
	\begin{align*}
		\R_n \leq \sqrt{8\EE[\bar \alpha_n \beta_n] \EE[\gamma_n] n} + \oO(1)\,,
	\end{align*}
	where $\bar \alpha_n = \frac{1}{n}\sum_{t=1}^n \alpha(\eE_t, z_t)$ is the average local alignment constant for the sequence of confidence sets $(\eE_t)_{t=1}^n$ realized by the algorithm.
\end{corollary}
\begin{proof}
	Along the lines of \cref{thm:regret-local}.
\end{proof}
In the contextual linear bandit setting, $\alpha(\eE_t, z) \leq 4$ holds for all $z \in \zZ$ (cf.~\cref{tab:examples}). Therefore, \cref{cor:local-upper-conditional} recovers the same bounds as the optimistic approach in the linear contextual bandit setting \citep[cf.,][]{abbasi2011improved}. Moreover, we immediately get a regret bound for the contextual extension of the dueling bandit setting in \cref{thm:regret-local}.

\subsection{Using the Context Distribution for Exploration}

Perhaps surprisingly, the contextual case allows for much weaker conditions under which the learner can achieve sublinear regret. This is possible if the learner exploits the distribution of contexts. Here we study the case where the context follows a fixed and known distribution $\zdistr \in \sP(\zZ)$. We only address the case where $\zdistr$ is known exactly. If the distribution is unknown, it is natural to replace $\zdistr$ with an online estimate of the context distribution \citep[cf.][]{tirinzoni2020asymptotically}. Extending the analysis to unknown context distribution is left as an important question for future work.

\begin{example}[Non-Informative Context]
	Consider the case where for some $z \in \zZ$ the learner obtains no information, i.e.~$M_a^z = 0$ for all $a \in \aA(z)$. In this case, the only sensible choice is the greedy action~$a^*(\hat \theta_t)$. The learner has to explore in rounds where information is available and the sampling distribution needs to be sufficiently diverse to account for rounds where the learner is forced to play greedily. Note that while there can be vanishing information gain in {\em some} rounds, the {\em expected} information gain with respect to the context distribution $\zdistr$ is non-zero. A natural application is in customer surveys: Clients who agree to provide feedback can be asked specifically targeted questions, whereas feedback from other customers is never observed.
\end{example}
\begin{example}[Greedy Exploration]
	Another interesting case is when feedback from the optimal action in each context is sufficiently diverse to allow estimation of the parameter \emph{without} further exploration. In such cases, the greedy algorithm can be highly effective. This effect has been studied in the bandit literature before \citep{bastani2017mostly,hao2019adaptive}. Conceptually, one can think of the context as part of the action space, where the sampling distribution is imposed by the environment.
\end{example}

%

The conditional IDS distribution (\cref{eq:conditional-ids}) is independent of the context distribution. It is easy to see that it behaves sub-optimally in both examples, and the bounds in \cref{cor:global-upper-conditional,cor:local-upper-conditional} become vacuous when a context occurs where the information gain is zero for all actions. To exploit the randomness of the context in the regret bounds, we include the contextual distribution in the information ratio. By optimizing the joint distribution over action and context, we obtain an IDS algorithm that makes use of the contextual distribution for exploration.

\newcommand{\sPZ}{\sP_\times(\aA(\zZ))}

Denote by $\sPZ = \times_{z \in \zZ} \sP(\aA(z))$ the set of probability kernels that assign each context $z$ a sampling distribution over $\aA$. Suggestively we write $\xi(a|z) = \xi(a,z)$ for elements in $\sPZ$. For context distribution $\zdistr \in \sP(\zZ)$ and kernel $\xi \in \sPZ$, we extend the definition of the gap estimates in \cref{eq:gap-contextual} to
\begin{align}
	\hat \Delta_t(\xi, \chi) \eqdef \int_\zZ  \int_{\aA(z)} \hat \Delta_t(a,z) d\xi(a|z)d\chi(z) \,.\label{eq:gallery-gap-integrated}
\end{align}
As information gain we use \cref{eq:info-contextual} with the same convention that 
\[I_t(\xi, \zdistr) = \int_\zZ \int_{\aA(z)} I_t(a,z) d\xi(a| z) d\zdistr(z)\,.\]
For $a \in \aA(\zZ)$ we also write $\hat \Delta_t(a, \chi) = \int_\zZ \hat \Delta_t(a(z),z) d\chi(z)$ and $I_t(a, \chi)= \int_\zZ I_t(a(z),z) d\chi(z)$.
\emph{Contextual IDS} is defined to optimize the conditional distribution  $\xi \in \sP_\times(\aA(\zZ))$,
\begin{align}
	\xi_t \eqdef \argmin_{\xi \in \sPZ} \left\{\Psi_t(\xi, \chi) = \frac{\hat \Delta_t(\xi, \zdistr)^2}{I_t(\xi,\zdistr)} \right\}\,.\label{eq:contextual-ids}
\end{align}
The action $a_t \sim \xi_t^\IDS(\cdot | z_t)$ is sampled from the conditional distribution corresponding to the observed context $z_t$. 
In the joint minimization of the information ratio, the contextual distribution contributes to exploration and a smaller information ratio. By Jensen's inequality and \cref{lem:ratio-convexity}, the contextual information ratio is never worse then the conditional information ratio:
\begin{align*}
 \min_{\xi \in \sPZ} \Psi_t(\xi, \chi) \leq \min_{\xi \in \sPZ} \int_{\zZ}  \Psi_t(\xi(z),z) d\chi(z) = \int_{\zZ} \min_{\mu \in \sP(\aA(z))} \Psi_t(\mu, z) d\chi(z)
\end{align*}

\paragraph{Regret Bounds}
Before presenting the regret bounds, we extend the definition of the alignment constant. Let $\pP(\zZ) \eqdef \times_{z \in \zZ}{\pP(z)}$ be the set of functions that map a context $z$ to a Pareto optimal action $a(z) \in \pP(z)$.  For context distribution $\zdistr \in \sP(\zZ)$, the global alignment constant is
\begin{align}
	\alpha(\zdistr) \eqdef \max_{\omega: \zZ\rightarrow \RR^{d}}\max_{a,b \in \pP(\zZ)}  \min_{c \in \aA(\zZ)} \frac{\big(\int_\zZ \ip{\phi_{a(z)}^z - \phi_{b(z)}^z,\omega^z} d\zdistr(z)\big)^2}{\int_{\zZ} \int_{\zZ} \|M_{c(z')}^{z'} \omega^{z}\|^2  d\zdistr(z) d\zdistr(z')}\,.\label{eq:contextual-alignment-glob}
\end{align}

For the locally observable case, the definition of plausible maximizers is extended to the product space $\pP(\eE, \zZ) \eqdef \times_{z \in \zZ} \pP(\eE, z)$ and $\bar \pP_\eta(\eE, \zZ) \eqdef \times_{z \in \zZ} \bar \pP_\eta(\eE, z)$.
Correspondingly, the local alignment constant for $\eE \subset \Theta$ is 
\begin{align}
	\alpha_\eta(\eE, \zdistr) \eqdef \max_{\omega \in \RR^{d\times \zZ}}\max_{a,b \in \pP(\eE, \zZ)}  \min_{c \in \bar \pP_\eta(\eE, \zZ)} \frac{\big(\int_\zZ \ip{\phi_{a(z)}^z - \phi_{b(z)}^z,\omega^z} d\zdistr(z)\big)^2}{\int_{\zZ} \int_{\zZ} \|M_{c(z')}^{z'} \omega^{z}\|^2  d\zdistr(z) d\zdistr(z')}\,.\label{eq:contextual-alignment-loc}
\end{align}
Finally, $\alpha(\eE, \zdistr) \eqdef \inf_{\eta > 0} \eta^2 \alpha_\eta(\eE, \zdistr)$.

Note that, reassuringly, the alignment constant $\alpha(\eE, \zdistr)$ is finite as long as any direction $\phi_a^z - \phi_b^z$ for context $z$ and Pareto optimal actions $a,b \in \pP(z)$ can be estimated under \emph{some} context $z' \in \zZ$ that occurs with positive probability $\chi(z') > 0$. Specifically,
\begin{align*}
	\alpha_\eta(\eE, \zdistr) \leq  \max_{z \in \zZ, \omega \in \RR^d} \max_{a,b \in \pP(\eE, \zZ)}  \min_{z' \in \zZ, c \in \bar \pP_\eta(\eE, \zZ)} \frac{\ip{\phi_{a(z)}^z - \phi_{b(z)}^z,\omega}^2}{ \chi(z') \|M_{c(z')}^{z'} \omega\|^2}\,.
\end{align*}

%


%

\begin{theorem}\label{thm:regret-global-contextual}
	For fixed $\zdistr \in \sP(\zZ)$ and context sequence $(z_t)_{t=1}^n$ sampled independently from $\zdistr$, the regret of contextual IDS (\cref{eq:contextual-ids}) satisfies
	\begin{align*}
		\EE[\R_n] \leq n^{2/3} \left(54 \alpha(\chi) \EE[\beta_n ]\EE[\gamma_n]\right)^{1/3} + \oO(\Delta_{\max})\,,
	\end{align*}
	where the expectation is over the random context.
\end{theorem}
\begin{proof}
	The proof follows along the lines of \cref{thm:regret-global}. Define $\hat a_t \in \pP(\zZ)$ by $\hat a_t(z) = \argmax_{a \in \pP(z)} \ip{\phi_a^z,\theta_t}$. It follows from the same steps in the aforementioned theorem and the definition of $\alpha(\chi)$, 
	\begin{align*}
		\hat \Delta_t(\hat a_t, \zdistr)^2 &\leq \max_{b \in \pP(\zZ)}  \left(\int_{\zZ} \max_{\theta \in \eE_t} \ip{\phi_{b(z)}^z - \phi_{\hat a_t(z)}^z, \theta - \hat \theta_t} d\chi(z)\right)^2 \\
		&\leq \alpha(\chi) \max_{\theta \in \eE_t}\max_{c \in \aA(\zZ)} \int_\zZ \int_\zZ \|M_{c(z')}^{z'} (\theta^z - \hat \theta_t)\|^2 d\chi(z') d\chi(z)\\
		&\leq 4 \beta_t \max_{c \in \aA(\zZ)} \int_{\zZ} \log \det(\eye_m + M_c^z V_t^{-1} M_c^{z\T}) d\chi(z)\\
		&\leq 8 \max_{c \in \aA(\zZ)} \beta_t I_t(c, \chi)\,.
	\end{align*}
	The information ratio is bounded by optimizing the trade-off between the greedy action $\hat a_t$ and $c_t = \argmax_{c \in \aA(\zZ)} I_t(c, \zdistr)$. Similar to \cref{eq:Psi_3 bound}, we obtain
	\begin{align*}
	\Psi_{3,t}(\xi_t, \chi) \eqdef \frac{\hat \Delta_t(\xi_t, \chi)^3}{I_t(\xi_t,\chi)} \leq  \alpha(\chi) \beta_t( \Delta_{\max}+ 4B)
	\end{align*}
	The proof is concluded with \cref{thm:ids-regret-general} and bounding the estimation error.
\end{proof}
The result for the locally observable case is stated in the next theorem.
\begin{theorem}\label{thm:regret-local-contextual}
	For fixed $\zdistr \in \sP(\zZ)$ and context sequence $(z_t)_{t=1}^n$ sampled independently from $\zdistr$, the regret of contextual IDS (\cref{eq:contextual-ids}) satisfies,
	\begin{align*}
		\EE[\R_n] \leq \sqrt{8 \EE[\bar \alpha_n \beta_n] \EE[\gamma_n] n} + \oO(\Delta_{\max})\,,
	\end{align*}
	where $\bar \alpha_n = \frac{1}{n}\alpha(\eE_t, \chi)$ is the average local alignment constant for the sequence of confidence sets $(\eE_t)_{t=1}^n$ realized by the algorithm.
\end{theorem}
\begin{proof}
	Again, we generalize the proof of \cref{thm:regret-local}. Denote $\eE_t = \times_{z \in \zZ} \eE_t$. Let $c_t = \argmax_{a \in \bar \pP(\eE_t, \zZ)} I_t(a)$ be the most informative action in $\bar \pP_\eta(\eE_t, \zZ)$ for some $\eta > 0$. For the gap estimate of $c_t$ we find
	\begin{align*}
	\eta^{-1}	\hat \Delta_t(c_t) &\leq \max_{a,b \in \pP(\eE_t, \zZ)} \int_\zZ \max_{\theta \in \eE_t}\ip{\phi_a^z - \phi_b^z, \theta} d\chi(z) \leq \max_{a,b \in \pP(\eE_t)} \max_{\theta, \nu \in \eE_t(\zZ)} \int_\zZ \ip{\phi_a^z - \phi_b^z, \theta - \nu} d\chi(z) \,.
	\end{align*}
	Consequently we can bound the information ratio for the action $c_t$ as follows,
	\begin{align*}
		\hat \Delta_t(c_t)^2 &\leq  \eta^2\alpha_\eta(\eE_t, \chi)  \max_{\theta, \nu \in \eE_t(\zZ)} \int_\zZ \int_{\zZ} \|M_{c_t}^{z'}(\theta^z - \nu^z)\|^2 d\chi(z') d\chi(z) \\
		&\leq 2  \eta^2\alpha_\eta(\eE_t, \chi) \beta_t  \int_\zZ   \lambda_{\max}\big(M_{c_t(z)}^{z} V_t^{-1} M_{c_t(z)}^{z^\T} \big) d\chi(z) \leq 8 \eta^2 \alpha_\eta(\eE_t, \chi) \beta_t I_t(c_t, \chi)\,.
	\end{align*}
	We conclude that $\Psi_t(c_t,\chi) \leq 8 \inf_{\eta > 0}\eta^2 \alpha_\eta(\eE_t, \chi)\beta_t$ and the result follows.
\end{proof}

\paragraph{Computation} Note that optimizing the conditional distribution is computationally more demanding than optimizing the conditional information ratio. Since the information ratio is a convex function of the distribution (\cref{lem:ratio-convexity}), we can optimize the conditional distribution using standard convex solvers. A particularly simple implementation uses the Frank-Wolfe algorithm \citep{frank1956algorithm}, that only relies on solving linear functions over $\sPZ$. The contextual IDS algorithm with Frank-Wolfe is summarized in \cref{alg:ids-context}. The gradient is
\begin{align*}
	\nabla_\xi \Psi(\xi, \chi) = \Psi(\xi,\chi) \left(\frac{2 \hat \Delta(\cdot, \cdot)}{\hat \Delta_t(\xi, \chi)} - \frac{\hat I_t(\cdot, \cdot)}{I_t(\xi, \chi)}\right) \in \RR^{\aA \times \zZ}
\end{align*}
Convergence of Frank-Wolfe is guaranteed assuming that the gradient is Lipschitz  \citep[Theorem~1]{jaggi2013revisiting}. In this case, the iterates $\xi_t^{(k)}$ approach the exact IDS distribution $\xi_t$ at rate
\begin{align*}
	\Psi(\xi_t^{(k)}, \chi) - \Psi(\xi_t, \chi) \leq  \oO\left(\frac{1}{k+2}\right)\,.
\end{align*}

Unfortunately, smoothness of the gradient $\nabla_{\xi} \Psi_t(\xi, \chi)$ is not guaranteed in general, in particular if for some $\xi$ the information gain $I_t(\xi, \chi) \approx 0$ is vanishing. On the other hand, for our choice of information gain and gap estimate, the gradient is smooth around the IDS distribution $\xi_t$. This follows from $\hat \Delta_t(a) \geq \Omega(t^{-1/2})$ and the bound on the information ratio, which implies that the information gain is $I_t(\xi_t) \geq \Omega(1/t)$. More explicitly, we can ensure smoothness by bounding the information gain away from zero. Define $I_t^\epsilon(a,z) = I_t(a,z) + \epsilon$. Using $\epsilon=1/t$ ensures that the gradient is $\oO(1/t^2)$-Lipschitz while only marginally increasing the total information gain. This suggests that $K=t^2$ iterations of Frank-Wolfe suffice to obtain a good approximation of the contextual IDS distribution.


\begin{algorithm2e}[t]
	\KwIn{Action set $\aA$, context set $\zZ$, context distribution $\zdistr \in \sP(\zZ)$, action-context features $\phi_a^z$, feedback maps $M_a^z$.}
	\caption{Contextual IDS with Frank-Wolfe} \label{alg:ids-context}
	\For{$t=1,2,3, \dots, n$}{
		$\hat \Delta_{t}(a, c) \gets \max_{\theta \in \eE_t} \max_{b \in \aA(z)} \ip{\phi_b^z - \phi_a^z, \theta}$ \tcp*{gap-estimates}
		$I_t(a,z) \gets \frac{1}{2} \log \det(\eye_m + M_a^z V_t^{-1}M_a^{z,\T})$\tcp*{information gain}
		$\xi_t^{(1)}(a,z) \gets 1/|\aA|, \,\forall a \in \aA, z \in \zZ$\;
		\For{$k=1,\dots,t^2$}{
			$\bar \Delta^{(k)} \gets \sum_{z \in \zZ, a \in \aA} \chi(z) \xi_t^{(k-1)}(a,z) \hat \Delta_t(a,z)$\;
			$\bar I^{(k)} \gets \sum_{z \in \zZ, a \in \aA} \chi(z) \xi_t^{(k-1)}(a,z) I_t(a,z)$\;
			\tcp{Gradient $\nabla_{\xi}\Psi_t(\xi, \chi)|_{\xi=\xi_{t}^{(k-1)}}$, up to a positive factor:}
			$G^{(k)}(a,z) \gets 2 \chi(z) \hat \Delta_t(a,z) \bar \Delta^{(k)} \bar I^{(k)} - \chi(z) I_t(a,z) (\bar \Delta^{(k)})^2$\;
			\tcp{Frank-Wolfe step}
			$a^*(z) \gets \argmin_{a \in \aA} G^{(k)}(a,z)$, $\quad\forall z \in \zZ$\;
			$\xi_t^{(k)}(a,z) \gets (1-\frac{2}{k+2}) \xi_t^{(k-1)}(a,z) + \frac{2}{k+2} \chf{a = a^*(z)},\quad \forall a \in \aA, z \in \zZ$\;
		}
		\texttt{Observe context:} $z_t \sim \zdistr$\;
		\texttt{Sample action:} $a_t \sim \xi_t^{(k)}(\cdot, z_t)$\;
		Choose $a_t$, observe $y_t = \ip{M_{a_t}^{z_t}, \theta} + \epsilon_t$\;
	}
\end{algorithm2e}

\paragraph{Tighter Gap Estimates and Alignment Constants}
We can obtain a tighter definition of the gap estimates using integrated reward features $\phi_\xi^\zdistr \eqdef \int_\zZ \int_{\aA(z)} \phi_a^z d\xi(a| z) d\zdistr(z)$ and defining the gap estimates $\bar \Delta_t(\xi, \chi) \eqdef \max_{\theta \in \eE_t} \max_{b \in \aA(\zZ)} \ip{\phi^\chi_{b} - \phi^\chi_\xi, \theta}$. By Jensen's inequality we get $\bar \Delta_t(\xi, \chi) \leq \hat \Delta_t(\xi, \chi)$. It is easy to verify that the definition of the alignment constant can be simplified and tightened for this choice of gap estimate: 
\begin{align}
\bar \alpha_\eta(\eE, \zdistr) \eqdef \max_{\omega \in \RR^d}\max_{a,b \in \pP(\eE, \zZ)}  \min_{c \in \bar \pP_\eta(\eE, \zZ)} \frac{\int_\zZ \ip{\phi_{a(z)}^z - \phi_{b(z)}^z,\omega}^2 d\zdistr(z)}{\int_{\zZ} \|M_{c(z)}^z \omega\|^2  d\zdistr(z)}\,.\label{eq:alpha-bar}
\end{align}
Jensen's inequality implies the natural property $\bar\alpha_\eta(\eE , \zdistr) \leq \int_\zZ \alpha_\eta(\eE, z) d\chi(z)$ where $\alpha(\eE, z)$ is the conditional alignment constant in \cref{eq:conditional-alignment}. Moreover, the definition is never worse than the contextual alignment constant in \cref{eq:contextual-alignment-loc}, i.e.~$\bar\alpha(\eE , \zdistr) \leq \alpha(\eE, \zdistr)$ where $\alpha(\eE, \zdistr)$. The downside is that computing the tighter gap estimate $\bar \Delta_t(\xi, \chi)$ appears to require a search over $\aA(\zZ)$, which leads computationally intractable methods in general.

\section{Kernelized Partial Monitoring}\label{sec:kernelized}
\newcommand{\fstar}{f^{\star}}

\looseness -1 Linear partial monitoring captures the relationship between reward and observation through linear reward features and linear feedback maps. 
Whereas we have focused on the finite-dimensional setting so far, we now consider the infinite-dimensional setting.
In particular, let $\hH$ be a Hilbert space over $\RR$ with norm $\|\cdot \|_{\hH}$ and inner product $\ip{\cdot, \cdot}_{\hH}$. We identify the unknown true reward function with a vector in $\fstar \in \hH$ that satisfies the known bound $\|\fstar\|_\hH \leq B$. The reward features $\phi_a \in \hH$ correspond to linear functionals and we adopt the notation $\fstar(a) = \ip{\phi_a, \fstar}_{\hH}$. Similarly, the feedback maps are modeled as linear operators $M_a : \hH \rightarrow \RR^m$ that observe the parameter on the subspace $\im(M_a^*) \subset \hH$, where $M_a^*$ is the adjoint mapping.
While this adds a great amount of flexibility, it poses two additional challenges.

First, computing the least-square estimate in the feature space requires $\Omega(d^2)$ memory and computation, which becomes prohibitive if $d$ is large or infinite. Kernelized methods circumvent this limitation using a \emph{representer theorem} \citep{kimeldorf1970correspondence,girosi1998equivalence,scholkopf2001generalized}. Such results express all quantities of interest as inner products that are specified by a known kernel function. Kernel methods are widely used in machine learning \citep{scholkopf2002learning}, and several kernelized bandit algorithms have been analyzed \citep{srinivas2009gaussian,abbasi2012online,valko2013kernelised,valko2014spectral,chaudhuri2016phased}. More broadly, by interpreting kernel regression as a Gaussian process \citep{rassmussen2004gp,kanagawa2018gaussian}, the field of Bayesian optimization is understood to solve a closely related problem \citep{mockus1982bayesian,srinivas2009gaussian}.

A second imminent issue that arises when the feature dimension is large or infinite, is that the dimension renders the previous results vacuous. In the literature on kernelized bandits, this challenge is circumvented by replacing the dimension by an appropriate notion of an \emph{effective dimension} \citep{valko2013kernelised}, or bounding the log-determinant using the eigendecay of the covariance matrix \citep{srinivas2009gaussian,vakili2020information}. 

\subsection{Kernel Regression for Partial Monitoring Feedback}

The least-squares estimate is defined over the Hilbert space $\hH$, using the observations $y_t = M_{a_t} \fstar + \epsilon_t$ and regularizer $\lambda > 0$,
\begin{align}
	\hat f_t \eqdef \argmin_{f \in \hH} \sum_{s=1}^{t-1} \|M_{a_s}f - y_s\|^2 + \lambda \|f\|_{\hH}^2\,. \label{eq:rkhs-ls}
\end{align}
For simplicity of the exposition, we present the version without additional parameter constraints. We remark that certain types of affine linear constraints can be handled by solving a finite-dimensional SOC-constrained problem \citep[e.g.][]{bagnell2015learning,aubin2021handling}
As usual, the regularized least-squares solution is always contained in a finite-dimensional subspace spanned by the data,
\begin{align*}
	\hat f_t \in \opspan{M_{a_s}^\T}{s \in [t-1]}\,.
\end{align*}
In other words, the least-squares solution can be parameterized by $\alpha_1, \dots, \alpha_{t-1} \in \RR^m$ such that $\hat f_t = \sum_{s=1}^{t-1} M_{a_s}^*\alpha_i$. We are interested in sufficient conditions to ensure that the coefficients and the evaluations maps $\ip{\phi_a, \hat f_t}$ can be computed efficiently. Define the \emph{joint evaluation map} for $a,b \in \aA$,
\begin{align}
	E_{a,b} : \hH \rightarrow \RR^{m+1},\quad f \mapsto [\phi_a f, (M_bf)^\T]^\T\,.\label{eq:gallery-J}
\end{align}
To enable efficient computation in the potentially infinite-dimensional Hilbert space $\hH$, we now assume that $E_{a,b}$ is an evaluation functional of a vector-valued reproducing kernel Hilbert space \citep[RKHS;][]{aronszajn1950theory,pedrick1957theory}. For a modern introduction to RKHS theory see \citep{carmeli2006vector}. 

\begin{assumption}[RKHS]\label{ass:rkhs}
	The subspace $\opspan{E_{a,b}^*}{a,b \in \aA}  \subset \hH$ is a $\RR^{m+1}$-valued RKHS over $\aA \times \aA$ with evaluation functionals $E_{a,b}$ defined in \cref{eq:gallery-J} and a known corresponding kernel \[k: \aA^2 \times \aA^2 \rightarrow \RR^{(m+1)\times(m+1)},\quad k(a,b,a',b') = E_{a,b}E_{a',b'}^*\,.\]
\end{assumption}
The assumption gives access to the covariance of actions $k_{\phi}(a,b) \eqdef \ip{\phi_a, \phi_b} \in \RR$, feedback $k_M(a,b) \eqdef M_a M_b^* \in \RR^{m \times m}$ and action-feedback $k_{\phi,M}(a,b) \eqdef \phi_a M_b^* \in \RR^{1 \times m}$. The next lemma shows that the reward estimate $\hat f_t(a) \eqdef \ip{\phi_a, \hat f_t}$ can be computed efficiently from finite-dimensional quantities.

\newcommand{\by}{{\bm y}}
\begin{theorem}[Partial Monitoring Representer Theorem]\label{lem:representer}
	Under \cref{ass:rkhs} the kernel least-squares estimate \cref{eq:rkhs-ls} is for any $a \in \aA$:
	\begin{align*}
		\hat f_t(a) = k_t(a)^\T (K_t + \lambda\eye_{m(t-1)})^{-1} \by_t\,,
	\end{align*}
	where we define the following finite-dimensional expressions:
	\begin{align*}
		\by_t &\eqdef [y_1^\T, \dots, y_{t-1}^\T]^\T \in \RR^{m(t-1)} &&\text{(the observation vector)}\\
		K_t &\eqdef [M_{a_r}M_{a_s}^*]_{r,s=1,\dots,t-1} \in \RR^{m(t-1) \times m(t-1)} &&\text{(kernel matrix)}\\
		k_t(a) &\eqdef [\phi_a M_{a_s}^*]_{s=1,\dots,t-1}^\T \in \RR^{m(t-1)} && \text{(evaluation weights)}
	\end{align*}
	In particular, the quantities above are defined by the kernel:
	\begin{align*} 
		k(a,b,a',b') \eqdef \begin{bmatrix}
			k_\phi(a,a') & k_{\phi,M}(a,b')\\
			k_{\phi,M}(a',b)^\T & k_M(b,b')
		\end{bmatrix} = \begin{bmatrix}
			\phi_a \phi_{a'}^* & \phi_a M_{b'}^*\\
			M_{b}\phi_{a'}^* & M_{b} M_{b'}^*\\
		\end{bmatrix}\,.
	\end{align*}
\end{theorem}
\begin{proof}
	We define the map 
	\begin{align}
		\Phi_t : \hH \rightarrow \RR^{m(t-1)}\,,\quad \theta \mapsto [(M_{a_1}\theta)^\T, \dots, (M_{a_{t-1}}\theta)^\T]^\T \label{eq:kernel-design-matrix}
	\end{align}
	as the stack of evaluation maps in the observation history. The regularized least-square solution of \cref{eq:rkhs-ls} is $\hat \theta_t = V_t^{-1} \Phi_t^* \by_t$, where $V_t \theta \eqdef (\Phi_t^*\Phi_t + \lambda \eye_\hH)\theta$ is an invertable linear map $\hH \rightarrow \hH$ and $\eye_{\hH}$ is the identity operator. The claim follows with the identity  $(\Phi_t^*\Phi_t + \lambda \eye_\hH)^{-1} \Phi_t^*$ = $\Phi_t^* (\Phi_t \Phi_t^* + \lambda\eye_t)^{-1}$ and replacing the inner products with the kernel expressions.
\end{proof}

In order to make use of the estimator in the IDS algorithm, we also need a kernelized statement of the confidence bounds. The next lemma directly extends the confidence bounds by \citet[Corollary 3.15]{abbasi2012online} to linear partial monitoring feedback.

\begin{lemma}\label{lem:ls-kernel}
	Let $V_t = \sum_{s=1}^{t-1} M_{a_s}^*M_{a_s} + \lambda\eye_{\hH}$ and $\eE_{t,\delta} = \{f \in \hH : \|f - \hat f_t\|_{V_t}^2 \leq \beta_{t,\delta}\}$ where $\beta_{t,\delta}^{1/2} = \rho\sqrt{2 \log \frac{1}{\delta} + \log \det (\eye + \lambda^{-1}K_t)} + \lambda ^{1/2} B$. Let $(a_t)_{t=1}^\infty$ be a $\fF_t$-adapted sequence of actions and corresponding observations $y_t = M_{a_t}\theta + \epsilon_t \in \RR^m$ with  conditionally independent $\rho$-sub-Gaussian vector $\epsilon_t$. If $\|f\|_{\hH} \leq B$, then
	\begin{align*}
		\PP[\forall t \geq 1, \fstar \in \eE_t] \geq 1 - \delta\,.
	\end{align*}
	Further, with probability at least $1-\delta$, for all $t \geq 1$,
	\begin{align*}
		|\hat f_t(a) - \hat f_t(b) - (\fstar(a) - \fstar(b))| = |\ip{\phi_a - \phi_b, \hat  f_t - \fstar}| \leq \sqrt{\beta_{t,\delta} \psi_t(a,b)}\,,
	\end{align*}
	where $\psi_t(a,b) \eqdef \frac{1}{\lambda}\Big(\psi(a,b) - (k_t(a) - k_t(b))^\T(K_t + \lambda \eye)^{-1}(k_t(a) - k_t(b))\Big)$ and the kernel metric is $\psi(a,b) \eqdef k_\phi(a,a) + k_\phi(b,b) - 2 k_\phi(a,b)$. The evaluation weights $k_t(a)$ and kernel matrix $K_t$ are defined in \cref{lem:representer}.
\end{lemma}
\begin{proof}
	The confidence set is the same as \citep[Corollary 3.15]{abbasi2012online} applied to the observation maps. For the second claim, note that $\psi_t(a,b) =\|\phi_a - \phi_b\|_{V_t^{-1}}^2$. The statement in the lemma follows using Cauchy-Schwarz and computing the feature uncertainty with the Sherman-Morrison identity,
	\[\lambda V_t^{-1} = \eye_\hH - \Phi_t^* (\Phi_t \Phi_t^* + \lambda \eye)^{-1} \Phi_t\,,\]
	where $\Phi_t$ is defined as in \cref{eq:kernel-design-matrix}.
\end{proof}

\subsection{Kernelized Information-Directed Sampling}\label{ss:gallery-kernelized-ids}

Equipped with the representer theorem and the kernelized confidence bound, we can define kernelized gap estimates and kernelized  information gain functions for information-directed sampling. We use the relaxed gap estimate from \cref{eq:gap-mean-def}, which is computationally simpler. Let $\hat f_t(a)$ as defined in \cref{lem:representer}, and $\beta_t \eqdef \beta_{t,1/t^2}$ and $\psi_t(a,b)$ as defined in \cref{lem:ls-kernel}. The kernelized gap estimate is
\begin{align}
	\hat \Delta_t(a) =  \min \big\{\max_{b \in \aA} \hat f_t(\hat a_t) + (\beta_t \psi_t(\hat a_t, b))^{1/2}  - \hat f_t(a) , B\big\}\,,\label{eq:kernel-gap}
\end{align}
where $\hat a_t = \argmax_{a \in \aA} \hat f_t(a)$ is the empirical maximizer. Other variants of the gap estimate are derived similarly. 
The (undirected) information gain corresponding to \cref{eq:info} is
\begin{align}
	I_t(a) = \sdfrac{1}{2} \log \det \left(\eye_m + \sdfrac{1}{\lambda}\big(k_M(a,a) - L_t(a) K_t^{-1} L_t(a)^\T\big)\right)\,,\label{eq:kernel-info}
\end{align}
where $L_t(a) = M_a\Phi_t^* \in \RR^{m \times (t-1)m}$ and $\Phi_t$ is the kernel design matrix defined in \cref{eq:kernel-design-matrix}. The total information gain is 
\begin{align*}
	\gamma_n = \frac{1}{2} \log \det (\eye + \lambda^{-1}K_{n+1})\,.
\end{align*}

\begin{corollary}
The theoretical guarantees for IDS \cref{thm:regret-global,thm:regret-local} stated in terms of the confidence coefficient $\beta_n$ and the total information gain $\gamma_n = \sum_{t=1}^n I_t(a_t)$ continue to hold up to constant factors.
\end{corollary}

The log-determinant is understood as a complexity measure of adaptive exploration in $\hH$ and is closely related to the notion of the Eluder dimension \citep{russo2013eluder,huang2021short}. In particular, $\gamma_n$ is often bounded independently of the dimension of $\hH$. For the large class of Mercer kernels, the literature has produced bounds depending on the decay of the eigenvalues in the Mercer decomposition \citep{srinivas2009gaussian,mutny2018efficient,vakili2020information}. A kernelized version of the directed  information gain \cref{eq:info-omega} can be derived similarly. 

\begin{example}[Kernelized Dueling Bandits]\label{ss:gallery-examples-dueling}
\newcommand{\muduel}{\mu^{\text{duel}}}
We present a kernelized version of the linear dueling bandit setting (\cref{ex:dueling}). As before, let $\iI$ be a ground set of actions. The action space $\aA = \iI \times \iI$ consists of pairs of elements in the ground set. In the \emph{utility-based dueling feedback} model, the reward and feedback is determined by a \emph{utility function} $g : \iI \rightarrow \RR$. Upon choosing the pair $a_t = (a_t^1, a_t^2) \in \aA$ in round $t$, and the learner observes the reward difference
\begin{align}
	y_t = g(a_t^1) - g(a_t^2) + \epsilon_t\,, \label{eq:gallery-dueling-feedback}
\end{align}
and suffers instantaneous regret $\fstar(a_t^1,a_t^2) = g(a_t^1) + g(a_t^2)$ for both actions. 

Let $\hH(\iI)$ be an RKHS with kernel function $k : \iI \times \iI$ and assume that the utility function $g \in \hH(\iI)$ satisfies $\|g\|_{\hH} \leq \frac{B}{2}$, therefore $\|\fstar\|_{\hH} \leq B$. For an action $a \in \iI$, denote by $k_a \in \hH(\iI)$ the kernel features of the evaluation functionals, which satisfy $k(a,b) = \ip{k_a, k_b}_\hH$. The features and evaluation maps corresponding to our reward and feedback model are $\phi_{a,b} = k_a + k_b$ and $M_{a,b} = k_a - k_b$. The covariance between reward and feedback for actions $a = (a^1,a^2)$ and $b = (b^1,b^2)$ is
\begin{align*}
	k_M(a,b) &= k(a^1,b^1) - k(a^2,b^1) - k(a^1,b^2) + k(a^2,b^2)\\
	k_{\phi,M}(a,b) &= k(a^1,b^1) + k(a^2,b^1) - k(a^1,b^2)  - k(a^2,b^2)\,.
\end{align*}
Hence, the kernel matrix and evaluation weights at time $t$ are
\begin{align*}
	K_t &= [k(a_r^1,a_s^1) - k(a_r^2,a_s^1) - k(a_r^1,a_s^2) + k(a_r^2, a_s^2)]_{r,s=1,\dots,t-1}\,,\\
	k_t(a) &= [k(a^1,a_s^1) + k(a^2,a_s^1) - k(a^1,a_s^2) - k(a^2, a_s^2)]_{s=1,\dots,t-1}\,.
\end{align*}

With the above, we can directly apply \cref{alg:ids-pm} and the corresponding results. A caveat is that the size of action space $|\aA|$ scales quadratically in the size of the ground set $|\iI|$. This leads to $\oO(|\iI|^2)$ computation complexity per round, even with the relaxed gap estimate \cref{eq:kernel-gap} and the approximate IDS distribution. This can be improved to $\oO(|\iI|)$, by directly estimating the utility function $g$ and using the structure of the dueling feedback \citep{kirschner2020dueling}. As usual, computing the kernel estimate in the data space requires $\oO(t^2)$ steps per round with incremental updates, or $\oO(n^3)$ overall on a horizon of length $n$.\looseness=-1
\end{example}

\section{Related Work}\label{sec:related}
Partial monitoring dates back to the work by \citet{rustichini1999games} and generalizes a considerable number of models for (stateless) sequential decision making that have been studied separately in the literature, most prominently the bandit setting \citep{lattimore2019bandit}. 
Like in the standard bandit model, one can consider the \emph{adversarial setting}, where the data is generated adversarially and the learner is compared to a fixed baseline. For a long time, the primary focus was on understanding the relationship between the minimax regret and the structure of the loss and feedback functions. The dependence on the horizon is now
completely understood as proven in a long line of work \citep{piccolboni2001discrete,cesa2006regret,antos2013towards,lattimore2019cleaning,lattimore2019information}.
An algorithm with rate-optimal worst-case regret in all classes of games is by \citet{lattimore2019exploration}.\looseness=-1

In the \emph{stochastic setting} the hidden outcomes are independent and identically distributed according to some unknown distribution \citep{bartok2011minimax}.
An algorithm for stochastic feedback that adapts to the game structure is by \citet{bartok2012adaptive}. Asymptotically optimal instance-dependent bounds were studied by \citet{komiyama2015regret}. 
For games that satisfy a local observability condition, \citet{vanchinathan2014efficient} analyze an algorithm that exploits prior knowledge on the loss distribution. They also propose a computationally more efficient variant based on Thompson sampling, which often has outstanding performance, while also suffering linear regret in certain games. 
More recently, \citet{tsuchiya2020analysis} derive logarithmic regret bounds for Thompson sampling on partial monitoring games that satisfy a strong local observability condition. 
Note, however, that for general partial monitoring games, Thompson sampling and algorithms based on optimism are insufficient to resolve the exploration-exploitation trade-off 
and might suffer linear regret \citep[Appendix G]{lattimore2019information}.
The setting with linear reward and feedback structure was first introduced by \citet{lin2014combinatorial} who provide an elimination-style algorithm that achieves $\tilde \oO(n^{2/3})$ regret under a global observability condition. A similar approach that achieves the same regret scaling is by \citet{chaudhuri2016phased}.

The information-directed sampling framework was first proposed by \citet{russo2014ids}, and later extended to the frequentist framework by \citet{kirschner2018heteroscedastic,kirschner2020pm}; extending instance optimality to other model classes is still an open problem. The latter work provides the basis of the current work. IDS was applied to the contextual linear bandit setting by \citet{hao2022contextualids}. In the adversarial setting, the IDS framework was used by \citet{lattimore2019information} to classify minimax rates and to derive an algorithm that applies to all finite games \citep{lattimore2019exploration}. For the linear bandit setting, \citet{kirschner2020asymptotically} show that for a specific choice of information gain function, IDS achieves the asymptotic lower bound while also being near worst-case optimal \citep{graves1997asymptotically,lattimore2017end,combes2017minimal}. Beyond the bandit setting, information-directed sampling was also applied to reinforcement learning \citep{nikolov2019information,lu2021reinforcement,hao2022regret,zanette2017information}. Various numerical results on the IDS approach can be found in \citep{russo2014ids,kirschner2020asymptotically,kirschner2020dueling,kirschner2021information}.\looseness=-1



\begin{table*}[t]
	\scriptsize
	\renewcommand{\arraystretch}{1.2}
	\begin{center}
				\begin{tabular}{|p{3cm}p{4.7cm}p{3.4cm}p{2.5cm}|}
			\hline
		\textbf{Algorithm} & \textbf{Reference} & \textbf{Theory} & \textbf{Compute} \\ \hline
			\textsc{FeedExp}         & \citet{piccolboni2001discrete}   & global, local, adversarial &\\ 
			\textsc{CBP} &   \citet{bartok2012adaptive}                 & global, local, frequentist & LP                               \\
			\textsc{NeighborhoodWatch} & \citet{foster2012no} & local, adversarial            &                                 \\
			\textsc{BPM-TS} & \citet{vanchinathan14efficient}            &                 --        & Gaussian sampling                \\
			\textsc{BPM-Least} & \citet{vanchinathan14efficient}             & local, frequentist            & SOCP                            \\
			\textsc{PM-DMED}                            & \citet{komiyama2015regret} &--    & LSIP+ECP$\dagger$  \\ 
			\textsc{PM-DMED-Hinge}                            & \citet{komiyama2015regret} & local, global, asymptotic   & LSIP+ECP$\dagger$  \\ 
			\textsc{Mario sampling} & \citet{lattimore2019information}    & local, Bayesian               & posterior sampling            \\
			\textsc{ExpByOpt}    & \citet{lattimore2019exploration}       & global, local, adversarial & ECP                              \\
			\textsc{TSPM} & \citet{tsuchiya2020analysis} & strongly local, frequentist &    \\
			\hline
			\textbf{\textsc{\IDS}}     & this work           & global, local, frequentist & SOCP                        \\ \hline
		\end{tabular} 
	\end{center}
	
	LP = linear programming, LSIP = linear semi-infinite program, SOCP = second-order cone programming, ECP = exponential cone programming
	
	$\dagger$ Alternatively, a convex/concave saddle-point problem that requires solving an ECP to evaluate.
	
	
	\caption{
		All algorithms need basic linear programming at initialization to determine estimation vectors and/or the cell decomposition.
		Algorithms with a blank compute entry can be computed using elementary matrix calculations only. 
		Most algorithms can be sped up, at least heuristically, by re-computing various quantities only intermittently.
	}\label{tab:algs}
\end{table*}

\paragraph{Tunable IDS and Estimation-To-Decisions}
\newcommand{\dec}{\text{dec}}
\newcommand{\gdec}{\text{g-dec}}
\newcommand{\ETD}{\text{E2D}}
The closely related estimation-to-decisions (E2D) framework was recently proposed by \cite{foster2021statistical}. The relation can be understood by introducing a \emph{gap-based decision-making coefficient} ($\gdec$):
\begin{align*}
	\mu_t^{\ETD} = \argmin_{\mu \in \sP(\aA)} \big\{\gdec_{\lambda}(\mu) \eqdef \hat \Delta_t(\mu) - \lambda I_t(\mu) \big\}
\end{align*}
We emphasize that our gap-based formulation of the decision-estimation coefficient is a relaxation of the formulation by  \citet{foster2021statistical}. Observe that the minimization is solved by a Dirac on $a_t= \argmin_{a \in \aA} \hat \Delta_t(a) - \lambda I_t(a)$. For a fixed $\lambda$, the gap-based E2D objective can be essentially solved by an offline oracle once the gap estimates and information gain function have been computed. The algebraic inequality $2ab \leq a^2 + b^2$ implies 
\begin{align*}
	\hat \Delta_t(\mu) - \lambda I_t(\mu) =  \frac{\hat \Delta_t(\mu) \sqrt{\lambda I_t(\mu)}}{\sqrt{\lambda I_t(\mu)}} - \lambda I_t(\mu) \leq \frac{\hat \Delta_t(\mu)^2}{4\lambda I_t(\mu)} =  \frac{\Psi_t(\mu)}{4\lambda}
\end{align*} 
Using this to bound the regret yields
\begin{align*}
	\sum_{t=1}^n \hat \Delta_t(a_t) &= \sum_{t=1}^n \big(\hat \Delta_t(a_t) - \lambda I_t(a_t)\big) + \lambda \sum_{t=1}^n I_t(a_t)\\
	&= \sum_{t=1}^n \min_{\mu \in \sP(\aA)} \big(\hat \Delta_t(\mu) - \lambda I_t(\mu)\big) + \lambda \gamma_n
	\leq \frac{1}{4\lambda} \sum_{t=1}^n \min_{\mu \in \sP(\aA)} \Psi_t(\mu) + \lambda \gamma_n
\end{align*}
By optimizing $\lambda$ and bounding the estimation error, we recover the IDS regret bound
\begin{align*}
	\R_n \leq \sqrt{\sum_{t=1}^n \min_{\mu \in \sP(\aA)} \Psi_t(\mu) \gamma_n} + \oO(1)
\end{align*}
In other words, the gap-based E2D algorithm with appropriately chosen $\lambda$ achieves the same worst-case bound as IDS. A similar argument can be made for the globally observable case.

A clear advantage of the gap-based E2D algorithm is that the optimal sampling distribution is realized as a Dirac. This is a significant simplification in the contextual setting, where E2D obtains the same bounds as contextual IDS without optimizing over the space of marginal distributions. The price for this simplification is that an optimal choice of $\lambda$ requires access to a bound on the information ratio and the information gain, and the algorithm becomes more dependent on choices that are informed by a worst-case analysis.

\section{Conclusion}

We presented \emph{linear partial monitoring}, a versatile framework for interactive decision making. Building upon and extending earlier work, we show that a single algorithm, information-directed sampling, achieves near-optimal regret rates in various settings, including parameter-constrained, kernelized and contextual decision-making problems. The framework includes the classical finite partial monitoring setting as a special case, and unlike at least some of the prior work, the proposed algorithm is simple practical to implement.
Promising directions for future work includes broadening the scope of information-directed sampling beyond the linear setting, deriving information gain functions for non-Gaussian observation likelihood functions and characterizing the exact minimax rate for continuous action sets.

\acks{Johannes Kirschner acknowledges funding through the SNSF Early Postdoc.Mobility fellowship P2EZP2\_199781. The work was supported by the European Research Council (ERC) under the European Union’s Horizon 2020 research and innovation programme grant aggreement No 815943. We thank the anonymous reviewers for valuable feedback.}


\newpage

\appendix

\section{Information-Directed Sampling: General Results}

\subsection{Properties of the Information-Ratio}
\begin{lemma}[Existence]\label{lem:ratio-existence}
	Let $\aA$ be compact, $\hDelta_t : \aA \rightarrow \RRp$ continuous and $I_t : \aA \rightarrow \RRp$ continuous and not zero everywhere. Then there exists a $\mu^* \in \sP(\aA)$ such that $\Psi_t(\mu^*) = \inf_{\mu \in \sP(\aA)} \Psi_t(\mu)$.
\end{lemma}
\begin{proof} The claim essentially follows from the fact that $\sP(\aA)$ is compact in the weak*-topology, which is also the topology that makes the maps $\mu \mapsto \hDelta_t(\mu)$ and $\mu \mapsto I_t(\mu)$ continuous. More specifically,
	%
	%
	pick a sequence $(\mu_j)_{j=1}^\infty$ in $\sP(\aA)$ such that $\Psi_t(\mu_j) \rightarrow \inf_{\mu \in \sP(\aA)} \Psi_t(\mu)$ as $j \rightarrow \infty$. 
	Note that $\mu_j$ is a tight sequence of probability distributions because $\aA$ is compact. Prokhorov's theorem \citep{prokhorov1956convergence} guarantees the existence of a subsequence $\mu_{j_i}$ converging weakly to some $\mu^* \in \sP(\aA)$. By definition of weak convergence of probability measures, $\hat \Delta_t(\mu_{j_i}) \rightarrow \hat \Delta_t(\mu^*)$ and $I_t(\mu_{j_i}) \rightarrow I_t(\mu^*)$. By the assumption that $I_t(\cdot)$ is not zero everywhere, we have $I_t(\mu^*) > 0$. Continuity of the map $(v,w) \mapsto v^2/w$ on $[0, \infty) \times (0, \infty)$ completes the proof.
\end{proof}


\begin{lemma}[Convexity {\citep[Proposition 6]{russo2014ids}}]\label{lem:ratio-convexity}
	$\Psi_t(\mu)$ is convex in $\mu$.
\end{lemma}
\begin{proof}
	Note that $(v,w) \mapsto v^2/w$ is convex on the domain $\RR \times (0, \infty)$ as shown in \citep[Chapter 3]{boyd2004convex}. Further, $\mu \mapsto (\hDelta_t(\mu), I_t(\mu))$ is an affine function on $\sP(\mu)$. Since $\Psi_t(\mu) = {\hat \Delta_t(\mu)^2}/{I_t(\mu)}$ can be written as a composition of a convex and an affine function, the result follows.
\end{proof}
The next lemma extends \citet[Prop.~6]{russo2014ids} to compact $\aA$.
\begin{lemma}[Support]\label{lem:ratio-support}
	The IDS distribution $\mu_t \in \argmin_{\mu \in \sP(\aA)} \Psi_t(\mu)$ can always be chosen such that $|\supp(\mu_t)| \leq 2$. 
	Further, for $a\in \aA$ define
	\begin{align*}
		h_t(a) \eqdef 2\; \hDelta_t(\mu_t) \hDelta_t(a)- \Psi_t(\mu_t) I_t(a)\,.
	\end{align*}
	Then any $a \in \supp(\mu_t)$ satisfies $h_t(a)= \min_{b \in \aA} h_t(b) = \hat \Delta_t(\mu_t)^2$.
\end{lemma} 
\begin{proof}
	We claim that 
	\begin{align}
		h_t(a) = \min_{b \in \aA} h_t(b) \quad \text{for all } a \in \supp(\mu_t)\,. \label{eq:proof-support-1}
	\end{align}
	We first show how this implies all other claims. Choose any minimizing distribution $\mu^* \in \argmin_{\mu \in \sP(\aA)} \Psi_t(\mu)$, not necessarily supported on two actions. Taking the expectation of $h_t(a)$ on $\mu^*$ gives $h_t(\mu^*) = \hDelta_t(\mu^*)^2$. Let $a_{\min} = \argmin_{a \in \supp(\mu^*)} \hDelta_t(a)$ and $a_{\max} = \argmax_{a \in \supp(\mu^*)}\hDelta_t(a)$. Then we can define $\mu^\IDS(p) = (1-p)\dirac{a_{\min}} + p \dirac{a_{\max}}$, where $\dirac{a}$ is a Dirac on $a \in \aA$ and $p \in [0,1]$ is a trade-off probability. We can choose $p^*$ such that $\hat \Delta_t(\mu^\IDS(p^*)) = \hat \Delta_t(\mu^*)$ and let $\mu_t = \mu_t(p^*)$. By \cref{eq:proof-support-1} we get $I_t(\mu_t) = I_t(\mu^*)$. Therefore $\Psi_t(\mu^*) = \Psi_t(\mu_t)$ and $\mu_t$ is a minimizing distribution with support size at most 2.

	To show \cref{eq:proof-support-1}, let $\Psi_t^* = \min_{\mu \in \sP(\aA)} \Psi_t(\mu)$ and define for $\mu \in \sP(\aA)$,
	\begin{align*}
		H_t(\mu) \eqdef \hDelta_t(\mu)^2  - I_t(\mu)\Psi_t^*\,.
	\end{align*}
	Note that $H_t$ has the same minimizers as $\Psi_t$. To see this,
	observe that $H(\mu) \geq 0$ and $H_t(\mu^*) = 0$, which shows one direction. For the converse, assume that $\mu'$ minimizes $H(\mu')$, i.e.\ $H_t(\mu') = 0$, which immediately gives $\Psi_t(\mu') = \Psi_t(\mu^*)$. 
	Let $a = \argmin_{b \in \supp(\aA)} h_t(b)$ which exists by compactness and continuity of $h$. Define the measure 
	$\mu_\lambda = (1-\lambda) \mu^* + \lambda \dirac{a}$ obtained from shifting mass to $a$.  Since $\mu^*$ is a minimizer of $H_t$, we must have that
	\begin{align*}
		0 \leq \frac{d}{d\lambda}H_t(\mu_\lambda)|_{\lambda = 0} &= 2 \hat \Delta_t(\mu^*) (\hDelta_t(\dirac{a}) - \hDelta_t(\mu^*)) - (I_t(\dirac{a}) - I_t(\mu^*))\\
		&=h_t(a) - h_t(\mu^*)\,.
	\end{align*}
	The claim follows after rearranging.
\end{proof}

\begin{lemma}[Closed form]\label{lem:ratio-closed-form}
	Let $0 < \Done \leq \Dtwo$ denote the gaps of two actions and $0 \leq \Ione, \Itwo$ the corresponding information gain. Define the ratio
	\begin{align*}
		\Psi(p) = \frac{\left((1-p) \Done + p \Dtwo\right)^2}{(1-p)\Ione + p \Itwo}\,.
	\end{align*}
	Then the optimal trade-off probability $p^* = \argmin_{0 \leq p \leq 1} \Psi(p)$ is
	\begin{align*}
		p^* = \begin{cases}
			0 & \text{if } \Ione \geq  \Itwo\\
			\text{\normalfont clip}_{[0,1]}\left(\frac{\Done}{\Dtwo- \Done} - \frac{2\Ione}{\Itwo - \Ione}\right) & \text{else,}
		\end{cases}
	\end{align*}
	where we define $\Delta_1/0 = \infty$ and $\text{\normalfont clip}_{[0,1]}(p) = \max(\min(p,1),0)$.
\end{lemma}
\begin{proof}
	The case $\Ione \geq \Itwo$ is immediate, because any $p > 0$ increases the numerator and decreases the denominator. For the remaining part we assume $\Ione < \Itwo$. The derivative is
	\begin{align*}
		\frac{d}{dp}\Psi(p) = \frac{\big(\Done + p(\Dtwo - \Done)\big)\big((\Dtwo -\Done) (2 \Ione + p(\Itwo - \Ione))- \Done(\Itwo- \Ione)\big)}{(\Ione + p(\Itwo-\Ione))^2}\,.
	\end{align*}
	\Cref{lem:ratio-convexity} implies that $\Psi(p)$ is convex on the domain $[0,1]$. Solving for the first order condition $\Psi'( p) = 0$ gives $\pzero \eqdef \frac{\Done}{\Dtwo- \Done} - \frac{2\Ione}{\Itwo - \Ione}$. If $\pzero \in [0,1]$ we are done. Otherwise, note that $\pzero < 0$ implies $\Psi'(0) > 0$ and $\pzero > 1$ implies $\Psi'(1) < 0$, which follows from calculating the sign of both factors in the nominator. Convexity on $[0,1]$ implies that clipping $\pzero$ to $[0,1]$ leads to the correct solution.
\end{proof}
We frequently use this lemma in the following way. Assume that $\tilde \mu \in \sP(\aA)$ is a sampling distribution, possibly chosen as a Dirac on some action $a \in \aA$. Let $\hat a_t = \argmin_{a \in \aA} \hDelta_t(a)$ be the action with the smallest estimated gap and denote $\delta_t = \hat \Delta_t(\hat a_t)$. Then
\begin{align*}
	\min_{\mu \in \sP(\aA)} \Psi_t(\mu) &\leq \min_{p \in [0,1]} \frac{\big((1-p)\delta_t + p \hat \Delta_t(\tilde \mu)\big)^2}{(1-p) I_t(\hat a_t) + p I_t(\tilde \mu)}\\
	&\leq \min_{p \in [0,1]} \frac{\big((1-p)\delta_t + p \hat \Delta_t(\tilde \mu)\big)^2}{p I_t(\tilde \mu)}\,.
\end{align*}
The first inequality is by restricting the sampling distribution as a mixture between a Dirac on $\hat a_t$ and $\tilde \mu_t$. The second inequality uses that $I_t(\hat a_t) \geq 0$. If we minimize the right-hand side using \cref{lem:ratio-closed-form}, we get
\begin{align}
	\Psi_t(\mu_t) \leq \begin{cases}
		\frac{4 \delta_t (\hat \Delta_t(\tilde \mu) - \delta_t)}{I_t(\tilde \mu_t)} & \text{if }2\delta_t \leq \hat \Delta_t(\tilde \mu)\\
		\frac{\hat \Delta_t(\tilde \mu)^2}{I_t(\tilde \mu)} & \text{else.}
	\end{cases}\label{eq:ratio-closed-form-remark}
\end{align}

\begin{lemma}[Almost greedy]\label{lem:ratio-greedy}
	Let $\hat a_t = \argmin_{a \in \aA} \hat \Delta_t(a)$ be the greedy action, chosen arbitrarily if not unique. The IDS distribution $\mu_t$ satisfies \[\hat \Delta_t(\mu_t) \leq 2 \hat \Delta_t(\hat a_t)\,.\]
\end{lemma}
\begin{proof}
	Note that by definition, the information ratio cannot be improved by shifting mass to $\hat a_t$ and discarding the information $I_t(\hat a_t)$,
	\begin{align*}
		\Psi_t(\mu_t) \leq  \min_{p \in [0,1]} \left\{\frac{\big((1-p) \hat \Delta_t(\mu_t) + p \hat \Delta_t(\hat a_t)\big)^2}{(1-p)I_t(\mu_t)} \eqdef \psi(p) \right\}\,.
	\end{align*}
	Note that the gradient of $\psi(p)$ cannot be negative at $p=0$. Hence
	\begin{align*}
		0 \leq \frac{d}{dp} \psi(p)|_{p=0} = \frac{2 \hat \Delta_t(\mu_t) \hat \Delta_t(\hat a_t) - \hat \Delta_t(\mu_t)^2}{I_t(\mu_t)}\,.
	\end{align*}
	Rearranging yields the claim.
\end{proof}

\begin{lemma}[Approximate IDS]\label{lem:ratio-approximate}
	Define the restricted set of sampling distributions $\sP_a = \{\dirac{a}(1-p) + \dirac{b}p : b \in \aA, p \in [0,1]\}$ that randomize between a fixed $a \in \aA$ and a second action $b \in \aA$. Let $\hat a_t = \argmin_{a \in \aA} \hat \Delta_t(a)$ be the greedy action. Define $\tilde \mu_t = \argmin_{\mu \in \sP_{\hat a_t}} \Psi_t(\mu)$ as the distribution that minimizes the information ratio among distribution in $\sP_{\hat a_t}$. Then 
	\begin{align*}
		\Psi_t(\tilde \mu_t) \leq \frac{4}{3} \min_{\mu \in \sP(\aA)}\Psi_t(\mu)\,,
	\end{align*}
	and the bound is tight for general $\Delta_t$ and $I_t$. Further, if $2 \hDelta_t(\hat a_t) \leq \hat \Delta_t(b)$ for all $b \in \aA$ with $\hat \Delta_t(b) > \hat \Delta_t(\hat a_t)$, then $\Psi_t(\tilde \mu_t) = \min_{\mu \in \sP(\aA)}\Psi_t(\mu)$.
\end{lemma}
\begin{proof}
	By \cref{lem:ratio-support} it suffices to consider three actions with gaps $\Delta_1 < \Delta_2 < \Delta_3$ and information gain $I_1, I_2, I_3$. Let $\Psi_{12}, \Psi_{13}$ and $\Psi_{23}$ denote the ratio obtained by minimizing the trade-off only between the actions indicated in the subscript. Assume that $\Psi^* \eqdef \Psi_{23} \stackrel{!}{=} \min\{\Psi_{12}, \Psi_{13}, \Psi_{23}\}$ and let $\tilde \Psi = \min\{\Psi_{12}, \Psi_{13}\}$. The claim follows if we show $\tilde \Psi \leq \frac{4}{3} \Psi^*$.
	
	Note that we can assume that $I_1=0$, since this choice does not affect $\Psi_{23}$ and can only make $\tilde \Psi$ larger. Further, the minimizer of the information ratio is invariant to rescaling of the gap and information gain functions. Therefore without loss of generality, we can assume that $\Delta_1=1$ and $\tilde \Psi=1$. 
	
	We show that $\Psi^{*-1} \leq \frac{4}{3}$.	First, we make some calculations with the help of \cref{lem:ratio-closed-form}. The trade-off probability between actions 2 and 3 is
	\begin{align*}
		p_{23} = \frac{\Delta_2}{\Delta_3 - \Delta_2} - \frac{2I_2}{I_3 - I_2}\,,
	\end{align*} 
	and we require that the trade-off is non-trivial, $0 < p_{23} < 1$. The ratio $\Psi_{23}$ is\looseness=-1
	\begin{align*}
		\Psi_{23} = \frac{4 (\Delta_2 I_3 - \Delta_3I_2)(\Delta_3 - \Delta_2)}{(I_3 - I_2)^2}\,.
	\end{align*}
	We complete the proof with two cases. For the first case, we assume $1 < \Delta_2 < \Delta_3 \leq 2$. We again use \cref{lem:ratio-closed-form} to compute $\Psi_{12} = {\Delta_2^2}/{I_2}$ and $\Psi_{13} = {\Delta_3^2}/{I_3}$. In fact, we can assume that $\Psi_{12} = \Psi_{13}$ since that does not affect $\tilde{\Psi}$ and only makes $\Psi^*$ smaller. The normalization $\tilde \Psi = 1$ implies that $I_2 = \Delta_2^2$ and $I_3 = \Delta_3^2$. Hence,
	\begin{align*}
		\Psi_{23}^{-1} &= \frac{(I_3 - I_2)^2}{4 (\Delta_2 I_3 - \Delta_3I_2)(\Delta_3 - \Delta_2)}\\
		&= \frac{(\Delta_3^2 - \Delta_2^2)^2}{4 (\Delta_2 \Delta_3^2 - \Delta_3\Delta_2^2)(\Delta_3 - \Delta_2)}\\
		&= \frac{(\Delta_3 + \Delta_2)^2}{4 \Delta_2 \Delta_3} \leq \frac{9}{8}\,.
	\end{align*}
	The last inequality holds for $1 \leq \Delta_2,\Delta_3 \leq 2$, and note that the constraint on $p_{23}$ is satisfied.
	
	For the second case, assume that $1 < \Delta_2 \leq 2 < \Delta_3$. In this case $\Psi_{13} = 4 (\Delta_3 - 1)$. The same normalization argument implies $I_2 = \Delta_2^2$ and $I_3 = 4(\Delta_3 - 1)$. With this, the ratio $\Psi_{23}$ is
	\begin{align*}
		\Psi_{23}^{-1} &= \frac{(I_3 - I_2)^2}{4 (\Delta_2 I_3 - \Delta_3I_2)(\Delta_3 - \Delta_2)}\\
		&= \frac{(4(\Delta_3-1) - \Delta_2^2)^2}{4 (4 \Delta_2 (\Delta_3 - 1) - \Delta_3\Delta_2^2)(\Delta_3 - \Delta_2)}\eqdef \varphi(\Delta_2, \Delta_3)\,.
	\end{align*}
	To eliminate $\Delta_3$ we compute the derivative 
	\begin{align*}
		\frac{d}{d\Delta_3} \varphi(\Delta_2, \Delta_3) = \frac{(\Delta_2 - 2)^3 (\Delta_2 - 2 \Delta_3 + 2) (-\Delta_2^2 + 4 \Delta_3 - 4)}{4 \Delta_2 (\Delta_2 - \Delta_3)^2 ((\Delta_2 - 4) \Delta_3 + 4)^2} > 0
	\end{align*}
	The inequality holds for all $1 \leq \Delta_2 \leq 2 < \Delta_3$. Hence it suffices to consider the limit 
	\begin{align*}
		\lim_{\Delta_3 \rightarrow \infty} \varphi(\Delta_2, \Delta_3) = \frac{4^2}{4(4\Delta_2 - \Delta_2^2)} = \frac{4}{\Delta_2(4 - \Delta_2)} \leq \frac{4}{3}\,.
	\end{align*}
	The last inequality holds for $1 \leq \Delta_2 \leq 2$ and the constraint on $p_{23}$ is satisfied. By \cref{lem:ratio-greedy}, $\Psi_{23}$ cannot be optimal if $\Delta_2 > 2 = 2 \Delta_1$. Finally, note that the bound is tight in the same limit. 
\end{proof}

\section{Additional Proofs}

\subsection{Proof of \cref{lem:confidence}}\label{proof:lem-confidence}

\begin{proof}
	The proof leverages the method of mixtures confidence set for the unconstrained least squares estimate \cite[c.f.][]{abbasi2011improved}, where we make the necessary adjustments to express the bound in terms of the basis $W$. 	Let $\hat \vartheta_t$ be the unconstrained least-square estimate,
	\begin{align*}
		\hat \vartheta_t = \argmin_{\vartheta \in \RR^d} \sum_{s=1}^{t-1} \|M_{a_s} \vartheta - y_s\|^2 + \lambda \|\vartheta - \theta_0\|^2 \,,
	\end{align*}
	Note that $\|M_{a_s} \vartheta - y_s\|^2 = \|M_{a_s} W W^\T (\vartheta - \theta^*) - \epsilon_s\|^2$. Hence the minimizer can be parametrized as $\hat \vartheta_t = (1_d-WW^\T)\theta_0 + W w$ for some $w \in \RR^d$.
	Further, $\hat \theta_t$ is the  projection of $\hat \vartheta_t$ onto the convex parameter set $\Theta$ with respect to the $\|\cdot\|_{V_t}$ norm. Therefore,
	\begin{align*}
		\|\htheta_t - \theta^*\|_{V_t} \leq \|\hat \vartheta_t - \theta^*\|_{V_t} = \|W^\T (\hat \vartheta_t - \theta^*)\|_{W_t}\,,
	\end{align*}
	where we used \cref{eq:W-def} in the second step to introduce the basis $W$. Define $\hat w_t = W^\T (\hat \vartheta_t - \theta^*) \in \RR^r$. Then, using that $y_s = M_{a_s} \theta^* + \epsilon_s$ and writing the least-squares objective directly in $\RR^r$, we find 
	\begin{align*}
		\hat w_t &= \argmin_{w \in \RR^r} \sum_{s=1}^{t-1} \|M_{a_s} W w - \epsilon_s\|^2 + \lambda \|W w + \theta^* - \theta_0\|^2\\
		&= (W V_t W)^{-1} \left(\sum_{s=1}^{t-1} W M_{a_s}^\T \epsilon_t - \lambda W (\theta^* - \theta_0)\right)
	\end{align*}
	Recall that $W_t = W^\T V_t W$. Using the closed-form to write the quantity of interest, we find
	\begin{align*}
		\|W^\T (\hat \vartheta_t - \theta^*)\|_{W_t} = \|\hat  w_t\|_{W_t} 
		&= \norm{\sum_{s=1}^{t-1} W M_{a_s}^\T \epsilon_t - \lambda W (\theta^* - \theta_0)}_{W_t^{-1}} \\
		&\leq \norm{\sum_{s=1}^{t-1} W M_{a_s}^\T \epsilon_t}_{W_t^{-1}} + \lambda \| W (\theta^* - \theta_0)\|_{W_t^{-1}}\\
		&\leq \norm{\sum_{s=1}^{t-1} W M_{a_s}^\T \epsilon_t}_{W_t^{-1}} + \sqrt{\lambda} \|\theta^* - \theta_0\|
	\end{align*}
	The claim follows now directly form \citet[Theorem 1]{abbasi2011improved}, and using that $ \lambda \| W (\theta^* - \theta_0)\|_{W_t^{-1}} \leq  \lambda \| W (\theta^* - \theta_0)\|_{W_0^{-1}} \leq \sqrt{\lambda} B$. 
\end{proof}

\subsection{Proof of \cref{lem:alpha-global}}\label{proof:alpha-global}
\begin{proof}[Sketch]
	Recall the definition of the observation matrix $S_a$ in Example~\ref{ex:finite}.
	By the definition of global observability, for any pair of Pareto optimal actions $a$ and $b$ there exist vectors $(w_{ab}^c)_{c \in [k]}$ in $\RR^m$ such that $\phi_a - \phi_b = \sum_{c \in [k]} S_c^\top w^c_{ab}$.
	By Proposition 37.18 \citep{lattimore2019bandit}, $w_{ab}^c$ can be chosen so that $\|w_{ab}^c\|_\infty \leq d^{1/2} k^{d/2}$. 
	The bound in \textit{(a)} follows from Cauchy-Schwarz to bound $\|w_{ab}^c\| \leq \sqrt{m} \|w_{ab}^c\|_\infty$ 
	and the second part of \cref{lem:alignment-bounds}. 
	Part \textit{(b)} follows in an identical fashion. For \textit{(c)} we need the concept of neighbours.
	Pareto actions $a$ and $b$ are neighbours if $\dim(\cC_a \cap \cC_b) = d-2$.
	Given neighbours $a$ and $b$ let $\theta_\circ$ be in the relative interior of $\cC_a \cap \cC_b$ and $\eE' = \{\theta : \norm{\theta - \theta_\circ} \leq \epsilon\}$ where $\epsilon > 0$ is small enough that $\eE' \subset \cC_a \cup \cC_b$.
	By the definition of local observability applied to $\eE'$ there exist a vectors $w_{ab}^a, w_{ab}^b \in \RR^m$ such that 
	$\phi_a - \phi_b = S_a^\top w^a_{ab} + S_b^\top w^b_{ab}$ and by Proposition 37.18 \citep{lattimore2019bandit}, 
	these vectors can be chosen so that $\norm{w^a_{ab}}_\infty \leq m$ and $\norm{w^b_{ab}}_\infty \leq m$.
	Finally, let $a$ and $b$ be arbitrary actions in $\bar \pP(\eE)$ that need not be neighbours.
	Then there exists a sequence $c_1,\ldots,c_j \in \bar \pP(\eE)$ with $j \leq k$ and $c_1 = a$ 
	and $c_j = b$ such that $c_i$ and $c_{i+1}$ are
	neighbours.
	Therefore, 
	\begin{align*}
		\phi_a - \phi_b = \sum_{i=1}^{j-1} \left(S_{c_i} w^{c_i}_{c_i c_{i+1}} + S_{c_{i+1}} w_{c_i c_{i+1}}^{c_{i+1}}\right) \,.
	\end{align*}
	By \cref{lem:alignment-bounds},
	\begin{align*}
		\alpha(\eE) &\leq \left(\sum_{i=1}^{j-1} \norm{w^{c_i}_{c_i c_{i+1}}} + \norm{w^{c_{i+1}}_{c_i c_{i+1}}}\right)^2 
		\leq 4k^2 m^3\,,
	\end{align*}
	which establishes \textit{(c)}.
\end{proof}

\vskip 0.2in
\bibliography{references}

\begin{thebibliography}{73}
\providecommand{\natexlab}[1]{#1}
\providecommand{\url}[1]{\texttt{#1}}
\expandafter\ifx\csname urlstyle\endcsname\relax
  \providecommand{\doi}[1]{doi: #1}\else
  \providecommand{\doi}{doi: \begingroup \urlstyle{rm}\Url}\fi

\bibitem[Abbasi-Yadkori(2012)]{abbasi2012online}
Yasin Abbasi-Yadkori.
\newblock \emph{Online Learning for Linearly Parametrized Control Problems}.
\newblock PhD thesis, 2012.

\bibitem[Abbasi-Yadkori et~al.(2011)Abbasi-Yadkori, P{\'a}l, and
  Szepesv{\'a}ri]{abbasi2011improved}
Yasin Abbasi-Yadkori, D{\'a}vid P{\'a}l, and Csaba Szepesv{\'a}ri.
\newblock Improved algorithms for linear stochastic bandits.
\newblock In \emph{Advances in Neural Information Processing Systems}, pages
  2312--2320, 2011.

\bibitem[Antos et~al.(2013)Antos, Bart{\'o}k, P{\'a}l, and
  Szepesv{\'a}ri]{antos2013towards}
A.~Antos, G.~Bart{\'o}k, D.~P{\'a}l, and Cs. Szepesv{\'a}ri.
\newblock Toward a classification of finite partial-monitoring games.
\newblock \emph{Theoretical Computer Science}, 473:\penalty0 77--99, 2013.

\bibitem[Aronszajn(1950)]{aronszajn1950theory}
Nachman Aronszajn.
\newblock Theory of reproducing kernels.
\newblock \emph{Transactions of the American mathematical society}, 68\penalty0
  (3):\penalty0 337--404, 1950.

\bibitem[Aubin-Frankowski and Szabo(2020)]{aubin2021handling}
Pierre-Cyril Aubin-Frankowski and Zoltan Szabo.
\newblock Hard shape-constrained kernel machines.
\newblock In H.~Larochelle, M.~Ranzato, R.~Hadsell, M.~F. Balcan, and H.~Lin,
  editors, \emph{Advances in Neural Information Processing Systems}, volume~33,
  pages 384--395. Curran Associates, Inc., 2020.
\newblock URL
  \url{https://proceedings.neurips.cc/paper/2020/file/03fa2f7502f5f6b9169e67d17cbf51bb-Paper.pdf}.

\bibitem[Bagnell and Farahmand(2015)]{bagnell2015learning}
J~Andrew Bagnell and Amir-massoud Farahmand.
\newblock Learning positive functions in a hilbert space.
\newblock In \emph{NIPS Workshop on Optimization (OPT2015)}, volume~20, pages
  3240--3255, 2015.

\bibitem[Bart{\'o}k et~al.(2011)Bart{\'o}k, P{\'a}l, and
  Szepesv{\'a}ri]{bartok2011minimax}
G.~Bart{\'o}k, D.~P{\'a}l, and Cs. Szepesv{\'a}ri.
\newblock Minimax regret of finite partial-monitoring games in stochastic
  environments.
\newblock In \emph{Proceedings of the 24th Annual Conference on Learning
  Theory}, pages 133--154, 2011.

\bibitem[Bart\'{o}k et~al.(2012)Bart\'{o}k, Zolghadr, and
  Szepesv\'{a}ri]{bartok2012adaptive}
G.~Bart\'{o}k, N.~Zolghadr, and Cs. Szepesv\'{a}ri.
\newblock An adaptive algorithm for finite stochastic partial monitoring.
\newblock In \emph{Proceedings of the 29th International Coference on
  International Conference on Machine Learning}, ICML, pages 1779--1786, USA,
  2012. Omnipress.

\bibitem[Bart{\'o}k et~al.(2014)Bart{\'o}k, Foster, P{\'a}l, Rakhlin, and
  Szepesv{\'a}ri]{bartok2014partial}
G.~Bart{\'o}k, D.~P. Foster, D.~P{\'a}l, A.~Rakhlin, and Cs. Szepesv{\'a}ri.
\newblock Partial monitoring---classification, regret bounds, and algorithms.
\newblock \emph{Mathematics of Operations Research}, 39\penalty0 (4):\penalty0
  967--997, 2014.

\bibitem[Bastani et~al.(2017)Bastani, Bayati, and Khosravi]{bastani2017mostly}
Hamsa Bastani, Mohsen Bayati, and Khashayar Khosravi.
\newblock Mostly exploration-free algorithms for contextual bandits.
\newblock \emph{arXiv preprint arXiv:1704.09011}, 2017.

\bibitem[Bengs et~al.(2021)Bengs, Busa-Fekete, El~Mesaoudi-Paul, and
  H{\"u}llermeier]{bengs2021preference}
Viktor Bengs, R{\'o}bert Busa-Fekete, Adil El~Mesaoudi-Paul, and Eyke
  H{\"u}llermeier.
\newblock Preference-based online learning with dueling bandits: A survey.
\newblock \emph{Journal of Machine Learning Research}, 22\penalty0
  (7):\penalty0 1--108, 2021.

\bibitem[Boyd et~al.(2004)Boyd, Boyd, and Vandenberghe]{boyd2004convex}
Stephen Boyd, Stephen~P Boyd, and Lieven Vandenberghe.
\newblock \emph{Convex optimization}.
\newblock Cambridge university press, 2004.

\bibitem[Carmeli et~al.(2006)Carmeli, De~Vito, and Toigo]{carmeli2006vector}
Claudio Carmeli, Ernesto De~Vito, and Alessandro Toigo.
\newblock Vector valued reproducing kernel hilbert spaces of integrable
  functions and mercer theorem.
\newblock \emph{Analysis and Applications}, 4\penalty0 (04):\penalty0 377--408,
  2006.

\bibitem[Caron et~al.(2012)Caron, Kveton, Lelarge, and
  Bhagat]{caron2012leveraging}
St{\'e}phane Caron, Branislav Kveton, Marc Lelarge, and Smriti Bhagat.
\newblock Leveraging side observations in stochastic bandits.
\newblock In \emph{Proceedings of the Twenty-Eighth Conference on Uncertainty
  in Artificial Intelligence}, pages 142--151, 2012.

\bibitem[Cesa-Bianchi et~al.(2006)Cesa-Bianchi, Lugosi, and
  Stoltz]{cesa2006regret}
N.~Cesa-Bianchi, G.~Lugosi, and G.~Stoltz.
\newblock Regret minimization under partial monitoring.
\newblock \emph{Mathematics of Operations Research}, 31:\penalty0 562--580,
  2006.

\bibitem[Cesa-Bianchi and Lugosi(2012)]{cesa2012combinatorial}
Nicolo Cesa-Bianchi and G{\'a}bor Lugosi.
\newblock Combinatorial bandits.
\newblock \emph{Journal of Computer and System Sciences}, 78\penalty0
  (5):\penalty0 1404--1422, 2012.

\bibitem[Cesa-Bianchi et~al.(2005)Cesa-Bianchi, Lugosi, and
  Stoltz]{cesa2005minimizing}
Nicolo Cesa-Bianchi, G{\'a}bor Lugosi, and Gilles Stoltz.
\newblock Minimizing regret with label efficient prediction.
\newblock \emph{IEEE Transactions on Information Theory}, 51\penalty0
  (6):\penalty0 2152--2162, 2005.

\bibitem[Chaudhuri and Tewari(2016)]{chaudhuri2016phased}
Sougata Chaudhuri and Ambuj Tewari.
\newblock Phased exploration with greedy exploitation in stochastic
  combinatorial partial monitoring games.
\newblock In \emph{Advances in Neural Information Processing Systems}, pages
  2433--2441, 2016.

\bibitem[Combes et~al.(2017)Combes, Magureanu, and
  Proutiere]{combes2017minimal}
Richard Combes, Stefan Magureanu, and Alexandre Proutiere.
\newblock Minimal exploration in structured stochastic bandits.
\newblock In \emph{Advances in Neural Information Processing Systems}, pages
  1763--1771, 2017.

\bibitem[Fiez et~al.(2019)Fiez, Jain, Jamieson, and
  Ratliff]{fiez2019transductive}
Tanner Fiez, Lalit Jain, Kevin~G Jamieson, and Lillian Ratliff.
\newblock Sequential experimental design for transductive linear bandits.
\newblock In \emph{Advances in Neural Information Processing Systems 32}, pages
  10666--10676. Curran Associates, Inc., 2019.

\bibitem[Foster and Rakhlin(2012)]{foster2012no}
Dean Foster and Alexander Rakhlin.
\newblock No internal regret via neighborhood watch.
\newblock In \emph{Artificial Intelligence and Statistics}, pages 382--390.
  PMLR, 2012.

\bibitem[Foster et~al.(2021)Foster, Kakade, Qian, and
  Rakhlin]{foster2021statistical}
Dylan~J Foster, Sham~M Kakade, Jian Qian, and Alexander Rakhlin.
\newblock The statistical complexity of interactive decision making.
\newblock \emph{arXiv preprint arXiv:2112.13487}, 2021.

\bibitem[Frank et~al.(1956)Frank, Wolfe, et~al.]{frank1956algorithm}
Marguerite Frank, Philip Wolfe, et~al.
\newblock An algorithm for quadratic programming.
\newblock \emph{Naval research logistics quarterly}, 3\penalty0 (1-2):\penalty0
  95--110, 1956.

\bibitem[Gajane and Urvoy(2015)]{gajane2015utility}
Pratik Gajane and Tanguy Urvoy.
\newblock Utility-based dueling bandits as a partial monitoring game.
\newblock \emph{arXiv preprint arXiv:1507.02750}, 2015.

\bibitem[Girosi(1998)]{girosi1998equivalence}
Federico Girosi.
\newblock An equivalence between sparse approximation and support vector
  machines.
\newblock \emph{Neural computation}, 10\penalty0 (6):\penalty0 1455--1480,
  1998.

\bibitem[Graves and Lai(1997)]{graves1997asymptotically}
Todd~L Graves and Tze~Leung Lai.
\newblock Asymptotically efficient adaptive choice of control laws incontrolled
  markov chains.
\newblock \emph{SIAM journal on control and optimization}, 35\penalty0
  (3):\penalty0 715--743, 1997.

\bibitem[Hao and Lattimore(2022)]{hao2022regret}
Botao Hao and Tor Lattimore.
\newblock Regret bounds for information-directed reinforcement learning.
\newblock \emph{arXiv preprint arXiv:2206.04640}, 2022.

\bibitem[Hao et~al.(2019)Hao, Lattimore, and Szepesvari]{hao2019adaptive}
Botao Hao, Tor Lattimore, and Csaba Szepesvari.
\newblock Adaptive exploration in linear contextual bandit.
\newblock \emph{arXiv preprint arXiv:1910.06996}, 2019.

\bibitem[Hao et~al.(2022)Hao, Lattimore, and Qin]{hao2022contextualids}
Botao Hao, Tor Lattimore, and Chao Qin.
\newblock Contextual information-directed sampling.
\newblock In Kamalika Chaudhuri, Stefanie Jegelka, Le~Song, Csaba Szepesvari,
  Gang Niu, and Sivan Sabato, editors, \emph{Proceedings of the 39th
  International Conference on Machine Learning}, volume 162 of
  \emph{Proceedings of Machine Learning Research}, pages 8446--8464. PMLR,
  17--23 Jul 2022.
\newblock URL \url{https://proceedings.mlr.press/v162/hao22b.html}.

\bibitem[Huang et~al.(2021)Huang, Kakade, Lee, and Lei]{huang2021short}
Kaixuan Huang, Sham~M Kakade, Jason~D Lee, and Qi~Lei.
\newblock A short note on the relationship of information gain and eluder
  dimension.
\newblock \emph{arXiv preprint arXiv:2107.02377}, 2021.

\bibitem[Jaggi(2013)]{jaggi2013revisiting}
Martin Jaggi.
\newblock Revisiting frank-wolfe: Projection-free sparse convex optimization.
\newblock In \emph{International Conference on Machine Learning}, pages
  427--435. PMLR, 2013.

\bibitem[Kanagawa et~al.(2018)Kanagawa, Hennig, Sejdinovic, and
  Sriperumbudur]{kanagawa2018gaussian}
Motonobu Kanagawa, Philipp Hennig, Dino Sejdinovic, and Bharath~K
  Sriperumbudur.
\newblock Gaussian processes and kernel methods: A review on connections and
  equivalences.
\newblock \emph{arXiv preprint arXiv:1807.02582}, 2018.

\bibitem[Kimeldorf and Wahba(1970)]{kimeldorf1970correspondence}
George~S Kimeldorf and Grace Wahba.
\newblock A correspondence between bayesian estimation on stochastic processes
  and smoothing by splines.
\newblock \emph{The Annals of Mathematical Statistics}, 41\penalty0
  (2):\penalty0 495--502, 1970.

\bibitem[Kirschner(2021)]{kirschner2021information}
Johannes Kirschner.
\newblock \emph{Information-Directed Sampling-Frequentist Analysis and
  Applications}.
\newblock PhD thesis, ETH Zurich, 2021.

\bibitem[Kirschner and Krause(2018)]{kirschner2018heteroscedastic}
Johannes Kirschner and Andreas Krause.
\newblock Information directed sampling and bandits with heteroscedastic noise.
\newblock In \emph{Proc. International Conference on Learning Theory (COLT)},
  July 2018.
\newblock URL \url{https://arxiv.org/abs/1801.09667}.

\bibitem[Kirschner and Krause(2021)]{kirschner2020dueling}
Johannes Kirschner and Andreas Krause.
\newblock Bias-robust bayesian optimization via dueling bandits.
\newblock In \emph{Proc. International Conference on Artificial Intelligence
  and Statistics (AISTATS)}, July 2021.

\bibitem[Kirschner et~al.(2020)Kirschner, Lattimore, and
  Krause]{kirschner2020pm}
Johannes Kirschner, Tor Lattimore, and Andreas Krause.
\newblock Information directed sampling for linear partial monitoring.
\newblock In \emph{Proc. International Conference on Learning Theory (COLT)},
  July 2020.
\newblock URL \url{https://arxiv.org/abs/2002.11182}.

\bibitem[Kirschner et~al.(2021)Kirschner, Lattimore, Vernade, and
  Szepesv{\'a}ri]{kirschner2020asymptotically}
Johannes Kirschner, Tor Lattimore, Claire Vernade, and Csaba Szepesv{\'a}ri.
\newblock Asymptotically optimal information-directed sampling.
\newblock In \emph{Proc. International Conference on Learning Theory (COLT)},
  August 2021.

\bibitem[Komiyama et~al.(2015)Komiyama, Honda, and
  Nakagawa]{komiyama2015regret}
J.~Komiyama, J.~Honda, and H.~Nakagawa.
\newblock Regret lower bound and optimal algorithm in finite stochastic partial
  monitoring.
\newblock In C.~Cortes, N.~D. Lawrence, D.~D. Lee, M.~Sugiyama, and R.~Garnett,
  editors, \emph{Advances in Neural Information Processing Systems 28}, NIPS,
  pages 1792--1800. Curran Associates, Inc., 2015.

\bibitem[Langford and Zhang(2008)]{langford2008epochgreedy}
J.~Langford and T.~Zhang.
\newblock The epoch-greedy algorithm for multi-armed bandits with side
  information.
\newblock In J.~C. Platt, D.~Koller, Y.~Singer, and S.~T. Roweis, editors,
  \emph{Advances in Neural Information Processing Systems 20}, NIPS, pages
  817--824. Curran Associates, Inc., 2008.

\bibitem[Lattimore and Gy{\"o}rgy(2020)]{lattimore2020mirror}
Tor Lattimore and Andr{\'a}s Gy{\"o}rgy.
\newblock Mirror descent and the information ratio.
\newblock \emph{arXiv preprint arXiv:2009.12228}, 2020.

\bibitem[Lattimore and Szepesv\'{a}ri(2017)]{lattimore2017end}
Tor Lattimore and Csaba Szepesv\'{a}ri.
\newblock The end of optimism? an asymptotic analysis of finite-armed linear
  bandits.
\newblock In \emph{Artificial Intelligence and Statistics}, pages 728--737,
  2017.

\bibitem[Lattimore and
  Szepesv{\'a}ri(2019{\natexlab{a}})]{lattimore2019cleaning}
Tor Lattimore and Csaba Szepesv{\'a}ri.
\newblock Cleaning up the neighborhood: A full classification for adversarial
  partial monitoring.
\newblock In \emph{Algorithmic Learning Theory}, pages 529--556,
  2019{\natexlab{a}}.

\bibitem[Lattimore and
  Szepesv{\'a}ri(2019{\natexlab{b}})]{lattimore2019exploration}
Tor Lattimore and Csaba Szepesv{\'a}ri.
\newblock Exploration by optimisation in partial monitoring.
\newblock \emph{arXiv preprint arXiv:1907.05772}, 2019{\natexlab{b}}.

\bibitem[Lattimore and
  Szepesv{\'a}ri(2019{\natexlab{c}})]{lattimore2019information}
Tor Lattimore and Csaba Szepesv{\'a}ri.
\newblock An information-theoretic approach to minimax regret in partial
  monitoring.
\newblock \emph{arXiv preprint arXiv:1902.00470}, 2019{\natexlab{c}}.

\bibitem[Lattimore and Szepesvari(2020)]{lattimore2019bandit}
Tor Lattimore and Czsaba Szepesvari.
\newblock \emph{Bandit Algorithms}.
\newblock Cambridge University Press, 2020.

\bibitem[Lin et~al.(2014)Lin, Abrahao, Kleinberg, Lui, and
  Chen]{lin2014combinatorial}
Tian Lin, Bruno Abrahao, Robert Kleinberg, John Lui, and Wei Chen.
\newblock Combinatorial partial monitoring game with linear feedback and its
  applications.
\newblock In \emph{International Conference on Machine Learning}, pages
  901--909, 2014.

\bibitem[Liu et~al.(2018)Liu, Buccapatnam, and Shroff]{liu2018information}
Fang Liu, Swapna Buccapatnam, and Ness Shroff.
\newblock Information directed sampling for stochastic bandits with graph
  feedback.
\newblock In \emph{Thirty-Second AAAI Conference on Artificial Intelligence},
  2018.

\bibitem[Lu et~al.(2021)Lu, Van~Roy, Dwaracherla, Ibrahimi, Osband, and
  Wen]{lu2021reinforcement}
Xiuyuan Lu, Benjamin Van~Roy, Vikranth Dwaracherla, Morteza Ibrahimi, Ian
  Osband, and Zheng Wen.
\newblock Reinforcement learning, bit by bit.
\newblock \emph{arXiv preprint arXiv:2103.04047}, 2021.

\bibitem[Mannor and Shamir(2011)]{mannor2011bandits}
Shie Mannor and Ohad Shamir.
\newblock From bandits to experts: On the value of side-observations.
\newblock In J.~Shawe-Taylor, R.~Zemel, P.~Bartlett, F.~Pereira, and K.~Q.
  Weinberger, editors, \emph{Advances in Neural Information Processing
  Systems}, volume~24. Curran Associates, Inc., 2011.

\bibitem[Mockus(1982)]{mockus1982bayesian}
Jonas Mockus.
\newblock The bayesian approach to global optimization.
\newblock \emph{System Modeling and Optimization}, pages 473--481, 1982.

\bibitem[Mutn\'y and Krause(2018)]{mutny2018efficient}
Mojmir Mutn\'y and Andreas Krause.
\newblock Efficient high dimensional bayesian optimization with additivity and
  quadrature fourier features.
\newblock In \emph{Neural and Information Processing Systems (NeurIPS)},
  December 2018.

\bibitem[Nikolov et~al.(2019)Nikolov, Kirschner, Berkenkamp, and
  Krause]{nikolov2019information}
Nikolay Nikolov, Johannes Kirschner, Felix Berkenkamp, and Andreas Krause.
\newblock Information-directed exploration for deep reinforcement learning.
\newblock In \emph{Proc. International Conference on Learning Representations
  (ICLR)}, May 2019.
\newblock URL \url{https://arxiv.org/abs/1812.07544}.

\bibitem[Pedrick(1957)]{pedrick1957theory}
George Pedrick.
\newblock \emph{Theory of reproducing kernels for Hilbert spaces of vector
  valued functions}.
\newblock PhD thesis, University of Kansas, 1957.

\bibitem[Piccolboni and Schindelhauer(2001)]{piccolboni2001discrete}
Antonio Piccolboni and Christian Schindelhauer.
\newblock Discrete prediction games with arbitrary feedback and loss.
\newblock In \emph{International Conference on Computational Learning Theory},
  pages 208--223. Springer, 2001.

\bibitem[Prokhorov(1956)]{prokhorov1956convergence}
Yu~V Prokhorov.
\newblock Convergence of random processes and limit theorems in probability
  theory.
\newblock \emph{Theory of Probability \& Its Applications}, 1\penalty0
  (2):\penalty0 157--214, 1956.

\bibitem[Rasmussen(2004)]{rassmussen2004gp}
Carl~Edward Rasmussen.
\newblock Gaussian processes in machine learning.
\newblock In \emph{Advanced lectures on machine learning}, pages 63--71.
  Springer, 2004.

\bibitem[Russo and Van~Roy(2013)]{russo2013eluder}
Daniel Russo and Benjamin Van~Roy.
\newblock Eluder dimension and the sample complexity of optimistic exploration.
\newblock \emph{Advances in Neural Information Processing Systems}, 26, 2013.

\bibitem[Russo and Van~Roy(2014)]{russo2014ids}
Daniel Russo and Benjamin Van~Roy.
\newblock Learning to optimize via information-directed sampling.
\newblock In \emph{Advances in Neural Information Processing Systems}, pages
  1583--1591, 2014.

\bibitem[Rustichini(1999)]{rustichini1999games}
A.~Rustichini.
\newblock Minimizing regret: The general case.
\newblock \emph{Games and Economic Behavior}, 29\penalty0 (1):\penalty0
  224--243, 1999.

\bibitem[Sch{\"o}lkopf et~al.(2001)Sch{\"o}lkopf, Herbrich, and
  Smola]{scholkopf2001generalized}
Bernhard Sch{\"o}lkopf, Ralf Herbrich, and Alex~J Smola.
\newblock A generalized representer theorem.
\newblock In \emph{International conference on computational learning theory},
  pages 416--426. Springer, 2001.

\bibitem[Sch{\"o}lkopf et~al.(2002)Sch{\"o}lkopf, Smola, Bach,
  et~al.]{scholkopf2002learning}
Bernhard Sch{\"o}lkopf, Alexander~J Smola, Francis Bach, et~al.
\newblock \emph{Learning with kernels: support vector machines, regularization,
  optimization, and beyond}.
\newblock 2002.

\bibitem[Srinivas et~al.(2010)Srinivas, Krause, Kakade, and
  Seeger]{srinivas2009gaussian}
Niranjan Srinivas, Andreas Krause, Sham~M Kakade, and Matthias Seeger.
\newblock Gaussian process optimization in the bandit setting: No regret and
  experimental design.
\newblock \emph{International Conference on Machine Learning}, 2010.

\bibitem[Sui et~al.(2018)Sui, Zoghi, Hofmann, and Yue]{sui2018advancements}
Yanan Sui, Masrour Zoghi, Katja Hofmann, and Yisong Yue.
\newblock Advancements in dueling bandits.
\newblock In \emph{IJCAI}, pages 5502--5510, 2018.

\bibitem[Tirinzoni et~al.(2020)Tirinzoni, Pirotta, Restelli, and
  Lazaric]{tirinzoni2020asymptotically}
Andrea Tirinzoni, Matteo Pirotta, Marcello Restelli, and Alessandro Lazaric.
\newblock An asymptotically optimal primal-dual incremental algorithm for
  contextual linear bandits.
\newblock \emph{Advances in Neural Information Processing Systems}, 33, 2020.

\bibitem[Tsuchiya et~al.(2020)Tsuchiya, Honda, and
  Sugiyama]{tsuchiya2020analysis}
Taira Tsuchiya, Junya Honda, and Masashi Sugiyama.
\newblock Analysis and design of thompson sampling for stochastic partial
  monitoring.
\newblock In H.~Larochelle, M.~Ranzato, R.~Hadsell, M.~F. Balcan, and H.~Lin,
  editors, \emph{Advances in Neural Information Processing Systems}, volume~33,
  pages 8861--8871. Curran Associates, Inc., 2020.
\newblock URL
  \url{https://proceedings.neurips.cc/paper/2020/file/649d45bf179296e31731adfd4df25588-Paper.pdf}.

\bibitem[Vakili et~al.(2020)Vakili, Khezeli, and
  Picheny]{vakili2020information}
Sattar Vakili, Kia Khezeli, and Victor Picheny.
\newblock On information gain and regret bounds in gaussian process bandits.
\newblock \emph{arXiv preprint arXiv:2009.06966}, 2020.

\bibitem[Valko et~al.(2013)Valko, Korda, Munos, Flaounas, and
  Cristianini]{valko2013kernelised}
Michal Valko, Nathaniel Korda, Remi Munos, Ilias Flaounas, and Nelo
  Cristianini.
\newblock Finite-{{Time Analysis}} of {{Kernelised Contextual Bandits}}.
\newblock \emph{arXiv:1309.6869 [cs, stat]}, September 2013.

\bibitem[Valko et~al.(2014)Valko, Munos, Kveton, and
  Koc{\'a}k]{valko2014spectral}
Michal Valko, R{\'e}mi Munos, Branislav Kveton, and Tom{\'a}{\v{s}} Koc{\'a}k.
\newblock Spectral bandits for smooth graph functions.
\newblock In \emph{International Conference on Machine Learning}, pages 46--54.
  PMLR, 2014.

\bibitem[Vanchinathan et~al.(2014{\natexlab{a}})Vanchinathan, Bart\'ok, and
  Krause]{vanchinathan14efficient}
Hastagiri Vanchinathan, G\'abor Bart\'ok, and Andreas Krause.
\newblock Efficient partial monitoring with prior information.
\newblock In \emph{Neural Information Processing Systems (NIPS)},
  2014{\natexlab{a}}.

\bibitem[Vanchinathan et~al.(2014{\natexlab{b}})Vanchinathan, Bart\'ok, and
  Krause]{vanchinathan2014efficient}
Hastagiri Vanchinathan, G\'abor Bart\'ok, and Andreas Krause.
\newblock Efficient partial monitoring with prior information.
\newblock In \emph{Neural Information Processing Systems (NIPS)},
  2014{\natexlab{b}}.

\bibitem[Yue et~al.(2012)Yue, Broder, Kleinberg, and Joachims]{yue2012k}
Yisong Yue, Josef Broder, Robert Kleinberg, and Thorsten Joachims.
\newblock The k-armed dueling bandits problem.
\newblock \emph{Journal of Computer and System Sciences}, 78\penalty0
  (5):\penalty0 1538--1556, 2012.

\bibitem[Zanette and Sarkar(2017)]{zanette2017information}
Andrea Zanette and Rahul Sarkar.
\newblock Information directed reinforcement learning.
\newblock Technical report, Technical report, Technical report, 2017.

\end{thebibliography}

\end{document}